\documentclass[11pt]{article}
\usepackage[margin=1in]{geometry}
\usepackage[T1]{fontenc}
\usepackage{times}
\usepackage{microtype}
\usepackage[authoryear,round]{natbib}

\usepackage{amsmath,amsfonts,bm}

\def\eqref#1{equation~\ref{#1}}

\def\1{\bm{1}}

\DeclareMathAlphabet{\mathsfit}{\encodingdefault}{\sfdefault}{m}{sl}
\SetMathAlphabet{\mathsfit}{bold}{\encodingdefault}{\sfdefault}{bx}{n}

\usepackage{graphicx}
\usepackage{amsmath,amssymb,amsthm}
\usepackage{mathtools}
\usepackage{bm}
\usepackage{booktabs}
\usepackage{subcaption}
\usepackage{enumitem}
\usepackage{algorithm}
\usepackage{algpseudocode}
\usepackage{flafter}
\usepackage{placeins}
\usepackage{wrapfig}

\usepackage[hidelinks]{hyperref}
\usepackage{url}
\usepackage[nameinlink,capitalize]{cleveref}
\makeatletter
\providecommand{\theHALG@line}{}
\renewcommand{\theHALG@line}{\thealgorithm.\arabic{ALG@line}}
\makeatother
\newtheorem{theorem}{Theorem}

\newtheorem{Remark}{Remark}

\title{Functional Autoencoders for Amplitude-Phase Representation Learning}
\author{
Peida Wu\thanks{These authors contributed equally. Wu began this work at IMS, ShanghaiTech University, before moving to his current affiliation.}\hspace{0.45em}$^{1}$,\quad
Xinyang Xiong\footnotemark[1]\hspace{0.45em}$^{2}$,\quad
Pengcheng Zeng\thanks{Corresponding author: \texttt{zengpch@shanghaitech.edu.cn}.}\hspace{0.45em}$^{2}$\\[0.6em]
{\small $^{1}$School of Engineering and Applied Science,}\\
{\small University of Pennsylvania, Philadelphia, PA, USA}\\[0.3em]
{\small $^{2}$Institute of Mathematical Sciences, ShanghaiTech University, Shanghai, China}
}
\date{}

\begin{document}
\maketitle
\raggedbottom

\begin{abstract}
Functional data are intrinsically infinite-dimensional, and often exhibit phase
variation, where corresponding events occur at different times across
observations. Existing linear dimension reduction methods struggle with
nonlinear amplitude variation, while functional autoencoders without an explicit
warp entangle temporal misalignment with shape. We propose the \textbf{Amplitude--Phase Functional Autoencoders (AP-FAE)}, an
unsupervised framework for functional data that spans both univariate and
multivariate cases, with emphasis on the multivariate setting, and factorizes
the latent space into separate amplitude and phase embeddings derived from all
channels. A
smooth functional decoder reconstructs channel-specific amplitude functions in
canonical time, and a shared monotone, endpoint-preserving warp captures phase
variation. We prove a bound linking amplitude recovery to registration,
reconstruction, and noise errors, and validate it numerically. Across synthetic
data and six real-world benchmarks, AP-FAE outperforms state-of-the-art
baselines on most clustering and alignment metrics and on all reconstruction
metrics. Clustering with amplitude embeddings alone consistently surpasses
joint amplitude--phase clustering, confirming the benefit of explicit
disentanglement. Code is available at
\href{https://anonymous.4open.science/r/APFAE-418C/}{https://anonymous.4open.science/r/APFAE-418C/}.
\end{abstract}

\section{Introduction}
\label{sec:introduction}

Functional data are functions over a continuous domain, such as
physiological signals, growth 
\begin{figure}[!tb]
    \centering
    \includegraphics[width=0.42\linewidth]{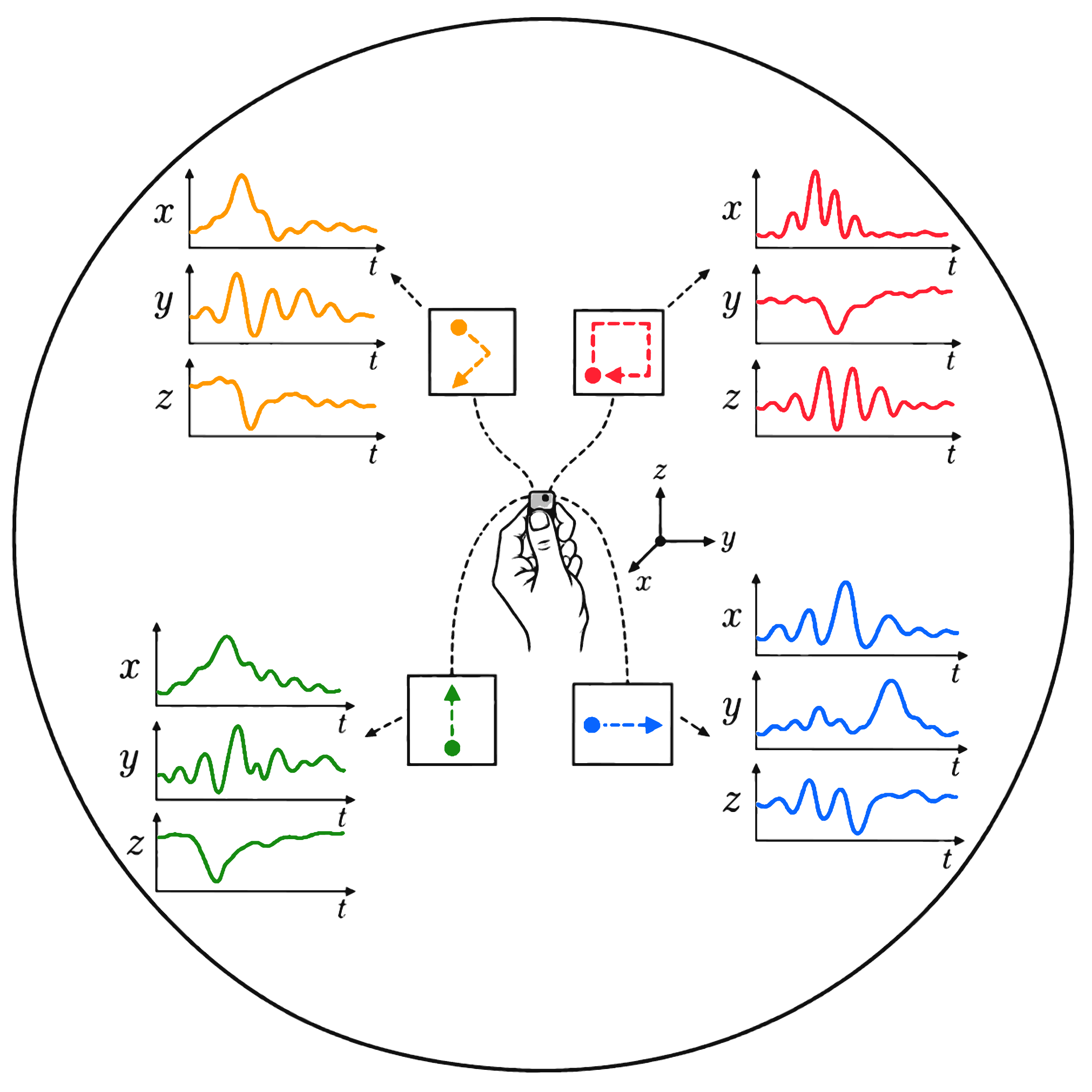}
\caption{Three synchronized accelerometer channels for four hand gestures (schematic).}
    \label{fig:uwave_multidimensional}
\end{figure} 
trajectories, or motion recordings
\citep{ramsay2005functional, Ferr2006}. In the multivariate setting, each observation is a vector-valued function
\(\mathbf{x}_i(t)=(x_{i1}(t),\ldots,x_{iD}(t))^\top\) with \(D\) synchronized
channels. For example, in \href{https://www.timeseriesclassification.com/description.php?Dataset=UWaveGestureLibrary}{UWaveGestureLibrary}, a hand-held accelerometer records
each gesture along the \(x\)-, \(y\)-, and \(z\)-axes, yielding three
synchronized channels (Figure~\ref{fig:uwave_multidimensional}).
Functional data analysis (FDA) exploits
smoothness and dependence structure for estimation, dimension reduction, and
learning \citep{Wang2015}. As sensing technologies produce increasingly rich
multichannel recordings, functional representation learning becomes important
for converting complex trajectories into compact features \citep{Wang2024functional}.

A first challenge is phase variation: corresponding events may occur at
different times, so curves with similar shapes can appear different
\citep{marron2015functional}. Ignoring this can make representations reflect
timing rather than shape \citep{kneip2008combining,tang2008pairwise,tucker2013generative}. Traditional registration estimates time
transformations before modeling \citep{srivastava2011registration}, but
alignment is then separate from the learning objective, and errors can
propagate. Joint approaches such as DeepFRC \citep{jiang2025deepfrc}, and NeuralFLoC \citep{xiong2026neuralfloc} learn
alignment with a modeling objective,
but their representations are tied to specific tasks and do not necessarily
provide separate compact amplitude and phase embeddings within a
reconstruction-based learner.

A second challenge is the intrinsic infinite-dimensionality of functional
data. Functional principal
component analysis (FPCA) addresses this by projecting onto a low-dimensional linear
subspace \citep{ramsay2005functional, Wang2015}, but linearity can be
restrictive for nonlinear variation, and channel-wise FPCA ignores
cross-channel covariance. Multivariate FPCA (MFPCA) provides joint components across channels but
remains linear \citep{happ2018multivariate}. Functional autoencoders (FAEs)
offer a nonlinear alternative with smooth functional reconstruction
\citep{HsiehSunWangHonavar2021, WuBeaulacCao2024}, and a joint multivariate encoder can aggregate
cross-channel information. Without an explicit warping mechanism to handle phase variation, however, FAEs must entangle temporal misalignment with amplitude variation.

To address these challenges, we propose the \textbf{Amplitude--Phase Functional
Autoencoders (AP-FAE)}, an unsupervised framework for temporally misaligned
functional data that applies to both univariate and multivariate settings, with
a primary focus on the multivariate case. A functional encoder jointly processes all
channels and separates the latent representation into amplitude and phase
embeddings. The amplitude decoder produces channel-specific trajectories in
canonical time, while the phase decoder generates a monotone,
endpoint-preserving warp shared across channels. AP-FAE is trained purely by
reconstruction, so the learned embeddings are not tied to any downstream task.
The amplitude and phase pathways are explicitly disentangled: the amplitude embedding
captures shape variation, and the phase embedding captures temporal deformation. This
design enables flexible downstream use, such as clustering on amplitude codes
alone or regression with both embeddings.

Our main contributions are summarized as follows:
\begin{itemize}
   \item To the best of our knowledge, we propose the first unsupervised
    framework that jointly learns nonlinear functional representations and
    explicitly disentangles amplitude and phase variation through separate
    embeddings and a shared monotone warp, with a primary focus on the
    multivariate setting.
    \item We establish the first theoretical bound quantifying how registration,
    reconstruction, and noise errors jointly control amplitude recovery in
    nonlinear functional representation learning, with numerical validation on
    synthetic data.
    \item We show that AP-FAE outperforms state-of-the-art baselines on most
    clustering and alignment metrics and on all reconstruction metrics across
    synthetic and six real-world benchmarks. Notably, on real data, clustering
    with amplitude embeddings alone consistently surpasses clustering with joint
    amplitude--phase embeddings, empirically confirming the benefit of explicit
    disentanglement for temporally misaligned functional data.
\end{itemize}

\section{Related Work}
\label{sec:related-work}

\paragraph{Functional alignment and amplitude--phase separation.}
Registration/Alignment estimates time transformations to separate amplitude from phase
\citep{marron2015functional}. Fisher--Rao and square-root velocity frameworks
provide geometric elastic alignment \citep{srivastava2011registration}, and
neural methods learn registration mappings \citep{Chen2021}. When used as
preprocessing, alignment is separate from representation learning. Joint methods
integrate alignment with clustering or classification.
For example, JPCCA \citep{gaffney2004jpcca} combines probabilistic alignment with model-based curve
clustering, SRC \citep{zeng2019src} jointly models registration and clustering, DeepFRC \citep{jiang2025deepfrc} connects
registration to supervised classification, and NeuralFLoC \citep{xiong2026neuralfloc} uses neural flows for
joint registration and clustering. These methods demonstrate the benefits of
integrating alignment with a modeling objective, but their representations are
tied to specific tasks. AP-FAE instead combines separate amplitude and phase embeddings with smooth
multivariate functional decoding under a reconstruction-only objective.
Clustering is applied to amplitude embeddings after fitting and does not
influence alignment.

\paragraph{Linear functional dimension reduction.}
FPCA provides optimal linear dimension reduction under squared-error loss
\citep{ramsay2005functional, Wang2015}, but may require many components for
nonlinear or misaligned curves, and channel-wise FPCA omits cross-channel
covariance. MFPCA extends to vector-valued observations with within- and
cross-component covariance \citep{happ2018multivariate}, giving a principled
multichannel basis. However, standard MFPCA is linear and lacks an explicit
subject-specific warp. Nonlinear extensions of FPCA exist but often rely on
kernel or manifold assumptions\citep{chen2012nonlinear,song2021nonlinear,tan2024nonlinear}, and they do not jointly model temporal
deformation. AP-FAE instead learns nonlinear amplitude and phase embeddings
under a reconstruction objective, providing a flexible alternative to linear
subspace methods.

 \paragraph{Autoencoders and functional autoencoders.}
Autoencoders learn low-dimensional representations through a reconstruction
bottleneck \citep{hinton2006reducing}. FAEs impose functional smoothness via basis representations of
filters and decoded trajectories \citep{HsiehSunWangHonavar2021,WuBeaulacCao2024}. Extending functional
projections across channels enables joint multivariate encoding \citep{chiou2014multivariate,happ2018multivariate}, but without
explicit temporal deformation, temporal shifts consume latent capacity and mix
with shape variation \citep{tucker2013generative,marron2015functional}. Some autoencoder variants incorporate temporal
regularization or recurrent structures \citep{srivastava2015unsupervised,malhotra2016lstm,
sagheer2019unsupervised}, but they do not separate amplitude from
phase. AP-FAE retains the smooth functional encoder--decoder while adding a
separate phase code and warping network, so temporal variation is modeled
explicitly rather than entangled with shape.

\section{The AP-FAE Framework}
\label{sec:model}

\paragraph{Overview.}
We consider $D$-channel functional data
$\mathbf{x}_i(t)=(x_{i1}(t),\ldots,x_{iD}(t))^\top$, where
$x_{id}\in L^2([0,1])$, $t\in[0,1]$, and $i=1,\ldots,N$. We assume the curves exhibit alignment/phase variation, which is common in
practice (as shown in the real-world data in Appendix~\ref{app:real-visualizations}),
and model them as
\begin{equation}
x_{id}(t)=f_{id}\bigl(\omega_i(t)\bigr)+\epsilon_{id}(t),
\qquad d=1,\ldots,D,
\label{eq:data-model}
\end{equation}
where the inverse warp $\omega_i$ maps observation time to a common canonical
time and is shared across channels, $f_{id}(t)$ is the amplitude function for
channel $d$, and $\epsilon_{id}(t)$ is measurement noise. Following classical
FDA formulations \citep{ramsay2005functional}, we assume $\omega_i$ is a
non-decreasing, continuous map on the compact domain $[0,1]$ with boundary
conditions $\omega_i(0)=0$ and $\omega_i(1)=1$; the domain is set to $[0,1]$
without loss of generality. Our goal is to learn nonlinear representations that disentangle amplitude variation in $f_{id}(t)$ from phase variation in $\omega_i(t)$,
yielding finite, compressed, and meaningful features for downstream modeling tasks such as regression, classification, and clustering.

As shown in Figure~\ref{fig:model_overview}, our framework jointly learns
a functional encoder $\mathcal{E}_\theta$ that maps $\mathbf{x}_i(t)$ to an amplitude embedding $\mathbf{a}_i\in\mathbb{R}^r$ and a phase embedding $\mathbf{p}_i\in\mathbb{R}^q$;
an amplitude decoder $\mathcal{D}_\phi$ that reconstructs amplitude functions $\{\widehat f_{id}\}_{d=1}^D$ from $\mathbf{a}_i$;
and a phase network $\mathcal{G}_\psi$ that reconstructs the inverse warp $\widehat\omega_i$ from $\mathbf{p}_i$.
Here $\widehat{\cdot}$ denotes estimated functions, and $\theta,\phi,\psi$ are the network parameters.
We minimize the reconstruction error
\begin{equation}
\theta,\phi,\psi
=
\operatorname*{arg\,min}_{\theta,\phi,\psi}
\frac{1}{ND}
\sum_{i=1}^N\sum_{d=1}^D
\int_0^1
\left\|x_{id}(t)-\widehat{x}_{id}(t)\right\|^2\,\mathrm{d}t,
\end{equation}
where $\widehat{x}_{id}(t)=\widehat{f}_{id}\bigl(\widehat{\omega}_i(t)\bigr)$.
The amplitude embedding $\mathbf{a}_i$ and phase embedding $\mathbf{p}_i$ separately encode amplitude and phase patterns.
This disentanglement is crucial for downstream modeling: for example, clustering with only $\{\mathbf{a}_i\}$ typically outperforms clustering with both $\{\mathbf{a}_i,\mathbf{p}_i\}$,
since the latter entangles phase information with amplitude patterns.
We call this framework AP-FAE, short for amplitude--phase functional autoencoders. We next describe its core modules step by step.

\begin{figure}[!htbp]
  \centering
  \includegraphics[width=\linewidth]{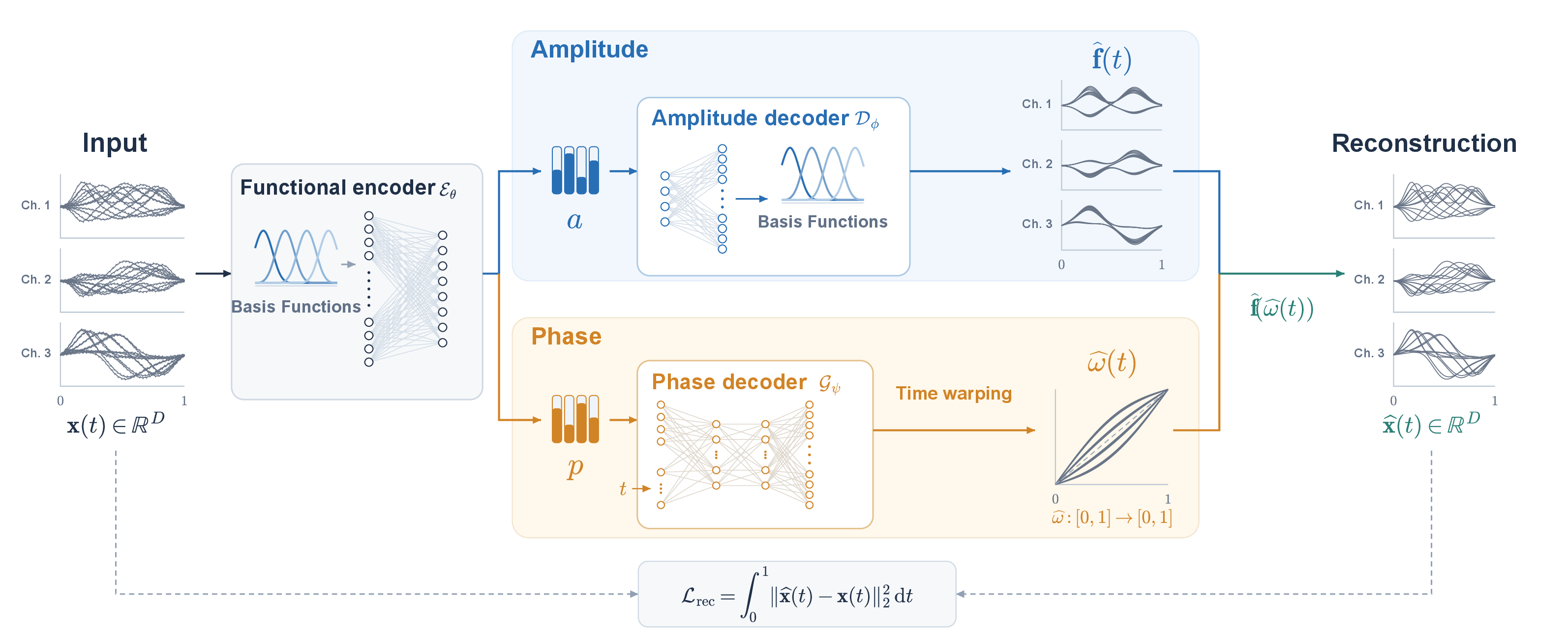}
\caption{Overview of AP-FAE. A functional encoder maps $\mathbf{x}(t)$ to amplitude and phase embeddings, decoded into channel-specific amplitude functions and a shared inverse warp; modules are trained jointly via reconstruction error to disentangle amplitude and phase variation. We use $D=3$ for illustration, with ``Ch.'' denoting ``Channel''.}
  \label{fig:model_overview}
\end{figure}

\paragraph{Functional encoder.}
Following classical functional autoencoders \citep{HsiehSunWangHonavar2021,WuBeaulacCao2024},
we encode $\mathbf{x}_i(t)$ into $K$ features $\mathbf{s}_i=(s_{i1},\ldots,s_{iK})^\top$
using functional weights and inner products:
\begin{equation}
s_{ik}
=
\sum_{d=1}^D \langle x_{id}(t),w_{kd}(t)\rangle
=
\sum_{d=1}^D \int_0^1 x_{id}(t)w_{kd}(t)\,\mathrm{d}t,
\qquad k=1,\ldots,K.
\label{eq:encoder-inner-product}
\end{equation}
Each weight function is expanded in $J$ fixed cubic B-spline bases $b_1(t),\ldots,b_J(t)$ on $[0,1]$:
\begin{equation}
w_{kd}(t)=\sum_{j=1}^J \alpha_{kdj}b_j(t),
\qquad d=1,\ldots,D,
\label{eq:encoder-weight-functions}
\end{equation}
where $\alpha_{kdj}$ are trainable coefficients.
For discretely observed functional data, the inner products are evaluated by numerical quadrature on a common grid $0=t_1<\cdots<t_M=1$ \citep{doi:10.1137/18M1229353}. 
We then apply an element-wise ReLU $\sigma(z)=\max\{0,z\}$ followed by an affine map:
\begin{equation}
\mathcal{E}_\theta(\mathbf{x}_i)
=
W_1\sigma(\mathbf{s}_i)+\mathbf{b}_1
=
\begin{pmatrix}
\mathbf{a}_i\\
\mathbf{p}_i
\end{pmatrix},
\label{eq:encoder-map}
\end{equation}
where $W_1\in\mathbb{R}^{(r+q)\times K}$ and $\mathbf{b}_1\in\mathbb{R}^{r+q}$ are trainable.
The first $r$ elements form $\mathbf{a}_i$, and the remaining $q$ elements form $\mathbf{p}_i$.
Both embeddings use information from all $D$ channels.
The encoder parameters are $\theta=\{\alpha_{kdj},W_1,\mathbf{b}_1\}$.

\paragraph{Amplitude decoder.}
The amplitude decoder maps $\mathbf{a}_i$ to amplitude functions on the canonical time scale:
$\mathcal{D}_\phi(\mathbf{a}_i)=(\widehat f_{i1}(t),\ldots,\widehat f_{iD}(t))^\top$.
Using the same fixed B-spline bases as the encoder,
\begin{equation}
\widehat f_{id}(t)=\sum_{j=1}^J c_{idj}b_j(t),
\qquad
c_{idj}=\sum_{k=1}^K \beta_{dkj}v_{ik},
\qquad d=1,\ldots,D,
\end{equation}
where
\begin{equation}
\mathbf{v}_i=(v_{i1},\ldots,v_{iK})^\top =\sigma(W_2\mathbf{a}_i+\mathbf{b}_2)\in\mathbb{R}^K.
\end{equation}
Here $W_2\in\mathbb{R}^{K\times r}$, $\mathbf{b}_2\in\mathbb{R}^K$, and $\beta_{dkj}$ are trainable.
Different channels have channel-specific coefficients but share the same amplitude embedding.
The phase embedding $\mathbf{p}_i$ is not used by this decoder.
The amplitude-decoder parameters are $\phi=\{\beta_{dkj},W_2,\mathbf{b}_2\}$.

\paragraph{Phase decoder and time warping.}
To reconstruct raw functional data on the observation time scale,
we need the inverse warp $\widehat\omega_i(t)$.
We use a multilayer perceptron $\mathcal{G}_\psi$ that takes a time coordinate and the phase code $(t,\mathbf{p}_i)$ as inputs, and define
\begin{equation}
\widehat\omega_i(t)
=
\frac{\int_0^t \exp\{\mathcal{G}_\psi(u,\mathbf{p}_i)\}\,\mathrm{d}u}
{\int_0^1 \exp\{\mathcal{G}_\psi(u,\mathbf{p}_i)\}\,\mathrm{d}u}.
\label{eq:normalized-integrated-exponential}
\end{equation}
The exponentiation ensures a positive warp density, so $\widehat\omega_i$ is strictly increasing and preserves the ordering of observed time points.
Normalization fixes the endpoints $\widehat\omega_i(0)=0$ and $\widehat\omega_i(1)=1$.
Larger values of $\widehat\omega_i(t)$ represent faster progress along the canonical time scale.
The network parameters $\psi$ are shared across subjects, while $\mathbf{p}_i$ allows subject-specific warps.
For a given subject, the same $\widehat\omega_i$ acts on every channel, preserving the common time coordinate across channels.
Algorithm~\ref{alg:warp} in Appendix~\ref{app:algorithms} summarizes the numerical computation of $\widehat\omega_i$; implementation details are provided in Appendix~\ref{app:implementation}.

\paragraph{Reconstruction and learning.}
We reconstruct the raw functional data by first mapping $t$ to canonical time $\widehat\omega_i(t)$ and then evaluating the decoded amplitude at $\widehat\omega_i(t)$:
\begin{equation}
\widehat x_{id}(t)
=
\widehat f_{id}\bigl(\widehat\omega_i(t)\bigr)
=
\sum_{j=1}^J c_{idj}b_j\bigl(\widehat\omega_i(t)\bigr),
\qquad d=1,\ldots,D.
\label{eq:twfae-decomposition}
\end{equation}
Equivalently, $\widehat x_{id}=\widehat f_{id}\circ\widehat\omega_i$, where $\circ$ denotes function composition.
Thus, $\mathbf{a}_i$ determines the curve shape through $\widehat f_{id}(t)$,
while $\mathbf{p}_i$ determines where that shape is evaluated in time through $\widehat\omega_i(t)$.
Using the common time grid $t_1,\ldots,t_M$, all parameters are learned jointly by minimizing
\begin{equation}
\mathcal{L}_{\mathrm{rec}}(\theta,\phi,\psi)
=
\frac{1}{NDM}
\sum_{i=1}^N\sum_{d=1}^D\sum_{m=1}^M
\left[x_{id}(t_m)-\widehat{x}_{id}(t_m)\right]^2.
\label{eq:twfae-loss}
\end{equation}
Algorithm~\ref{alg:twfae} in Appendix~\ref{app:algorithms} summarizes the joint updates using AdamW\citep{loshchilov2019decoupled}.
For downstream tasks, we use $\{\mathbf{a}_i\}$ for amplitude-based clustering,
and defer the use of both $\{\mathbf{a}_i,\mathbf{p}_i\}$ to settings where
phase variation is predictive; the estimated functions
$\{\widehat{f}_{id},\widehat{\omega}_i\}$ can further visualize the learned
amplitude--phase disentanglement.

\paragraph{Computational and space complexity.}
Training AP-FAE scales linearly in the number of subjects $N$, channels $D$, grid points $M$, and B-spline basis size $J$.
The shared phase network avoids an additional factor of $D$ in the warping pathway.
Working memory during minibatch training is dominated by model and optimizer storage together with the activation cache,
while full-dataset storage requires $\mathcal{O}(NDM)$.
A complete derivation, including per-epoch complexity and memory breakdown, is given in Appendix~\ref{app:complexity}.

\section{Theory}
\label{sec:insight}

We bound the error in recovering the true amplitude curves by registration,
reconstruction, and measurement-noise errors. Let
\(\mathbf{x}_i=\mathbf{f}_i\circ\omega_i+\boldsymbol e_i\) denote the observed
curves and \(\widehat{\mathbf{x}}_i=\widehat{\mathbf{f}}_i\circ\widehat{\omega}_i\)
the fitted reconstructions, where 
\(\boldsymbol e_i\in L^2([0,1];\mathbb{R}^D)\) is measurement noise. For a
\(D\)-variate function \(\mathbf{u}\), define
\(\|\mathbf{u}\|_{L^2}^2=\sum_{d=1}^D\int_0^1 u_d(t)^2\,\mathrm{d}t\).

\begin{theorem}[Amplitude error bound]
\label{thm:amplitude-recovery-noise}
Assume that
(i) \(\|\mathbf{f}_i(s)-\mathbf{f}_i(t)\|_2\le L|s-t|\) for all \(i\) and
\(s,t\in[0,1]\);
(ii) the true and fitted inverse warps \(\omega_i,\widehat{\omega}_i\) are
strictly increasing continuous bijections of \([0,1]\) onto itself;
and (iii) each fitted inverse warp \(\widehat{\omega}_i\) is \(C\)-Lipschitz.
Define
\[
\epsilon_g=\max_i\|\widehat{\omega}_i-\omega_i\|_\infty,
\qquad
\epsilon_x^2=\frac1N\sum_{i=1}^N
\|\widehat{\mathbf{x}}_i-\mathbf{x}_i\|_{L^2}^2,
\qquad
\tau^2=\frac1N\sum_{i=1}^N\|\boldsymbol e_i\|_{L^2}^2.
\]
Then
\begin{equation}
\left\{\frac1N\sum_{i=1}^N
\|\widehat{\mathbf{f}}_i-\mathbf{f}_i\|_{L^2}^2\right\}^{1/2}
\le L\epsilon_g+\sqrt{C}\,(\epsilon_x+\tau).
\label{eq:amplitude-recovery-bound-noise}
\end{equation}
Here \(\epsilon_x\) is measured against the noisy observations. This
deterministic bound requires no independence or zero-mean assumption on the
noise.
\end{theorem}

\begin{Remark}[Practical implication and feasibility]
\label{rem:amplitude-recovery}
\Cref{thm:amplitude-recovery-noise} shows that amplitude recovery is controlled
by three separately measurable quantities: registration error \(\epsilon_g\),
reconstruction error \(\epsilon_x\), and noise level \(\tau\). Thus, small
reconstruction error alone does not guarantee accurate amplitude recovery;
reliable amplitude embeddings require accurate phase alignment. In practice,
\(\epsilon_x\) is directly minimized during training, \(\tau\) can be estimated
from repeated measurements or a noise model, and \(\epsilon_g\) can be monitored
on synthetic or held-out aligned data. The assumptions are mild for functional
registration: the fitted warps are strictly increasing by construction through
positive warp densities and endpoint normalization, and bounded-slope
piecewise-linear or smooth warps satisfy the \(C\)-Lipschitz condition directly.
In our simulations, the resulting bound is satisfied across all scenarios. As in
functional registration, amplitude and phase are identifiable only relative to a
common aligned coordinate; the theorem therefore conditions on registration
accuracy in that coordinate rather than deriving it from reconstruction loss
alone.
\end{Remark}

The proof is given in Appendix~\ref{app:amplitude-recovery}.

\section{Experiments}
\label{sec:experiments}

\paragraph{Setup.}
AP-FAE uses \(J=20\) cubic B-spline bases, functional hidden width \(K=16\),
amplitude and phase dimensions \(r=q=4\), and a time-conditioned warp MLP with
two hidden layers of width \(32\). Clustering applies \(K\)-means with \(20\)
initializations to the amplitude embeddings \(\{\mathbf{a}_i\}\). We evaluate
on synthetic data and six real datasets (channels \(1\)--\(8\), classes
\(3\)--\(8\)), reporting clustering accuracy (ACC), adjusted total variance
(ATV) for registration, and mean squared error (MSE) for reconstruction; higher
ACC and lower ATV/MSE are better. We compare against four baseline categories:
(i) \emph{classical functional autoencoders}, including FAE
\citep{HsiehSunWangHonavar2021}, which learns a functional representation without an
explicit warp, and FAEclust \citep{singh2025faeclust}, which combines functional
autoencoding with shape-based clustering;
(ii) \emph{neural functional registration}, represented by SrvfRegNet
\citep{Chen2021};
(iii) \emph{joint registration and clustering}, including JPCCA
\citep{gaffney2004jpcca} and SRC \citep{zeng2019src}, which estimate temporal
alignment together with cluster structure; and
(iv) \emph{traditional functional clustering}, including Funclust
\citep{JACQUES2013164} and funHDDC \citep{BouveyronJacques2011funHDDC}, which
model class structure through functional principal components and
group-specific functional subspaces. For fair comparison, all clustering
baselines use the true number of classes, and comparisons are partial: only
AP-FAE performs clustering, alignment, and reconstruction jointly. Full
protocols, metrics, and baseline implementations are given in
Appendices~\ref{app:evaluation-metrics} and 
\ref{app:simulation-protocol}.

\subsection{Simulation}
\label{sec:simulation}
\paragraph{Data generation.}
We generate three classes of trivariate curves with \(100\) subjects per class,
\(100\) grid points, class-specific three-harmonic amplitudes, and shared
subject-specific exponential warps independent of class. We vary the
cross-channel coefficient correlation \(\rho\in\{0.2,0.6\}\) and the
phase-distortion range \(B\in\{0.75,1.25\}\);
larger \(\rho\) induces stronger correlation among channels, while larger \(B\)
produces more severe phase distortion. The complete design is in
Appendix~\ref{app:simulation-design}.

\paragraph{Comparison with baselines.}
AP-FAE achieves the best ACC, ATV, and MSE in all four settings of
Table~\ref{tab:sim-comparison}, outperforming the strongest competitor SRC
(e.g., at \(\rho=0.2, B=0.75\): ACC \(0.924\) vs.\ \(0.875\), ATV \(1.187\)
vs.\ \(1.347\)). FAE, which lacks an explicit warp, degrades in clustering as
\(B\) increases (ACC \(0.690\!\to\!0.586\) at \(\rho=0.2\);
\(0.688\!\to\!0.572\) at \(\rho=0.6\)), and the alignment-only SrvfRegNet also
worsens (ATV \(1.502\!\to\!1.965\) at \(\rho=0.2\)), whereas AP-FAE remains
stable; these trends confirm that amplitude-based grouping benefits from joint
temporal modeling. Increasing \(\rho\) mainly lowers ACC while leaving ATV and
MSE nearly unchanged, indicating that the metrics capture distinct aspects of
the fit. Comparisons are partial: SrvfRegNet reports alignment only;
Funclust/funHDDC report clustering only; FAE and FAEclust report clustering and
reconstruction but not alignment; JPCCA and SRC report clustering and alignment
but not reconstruction. Thus MSE is compared against two autoencoder
references, FAE and FAEclust.

\begin{table}[!htbp]
\centering
\small
\setlength{\tabcolsep}{3.2pt}
\renewcommand{\arraystretch}{1.05}
\resizebox{\linewidth}{!}{%
\begin{tabular}{llcccccccc}
\toprule
Settings & Metric & Funclust & funHDDC & JPCCA & SRC & SrvfRegNet & FAE & FAEclust & AP-FAE \\
\midrule
$\rho=0.2, B=0.75$ & ACC $\uparrow$   & 0.665 & 0.679 & 0.585 & \underline{0.875} & -- & 0.690 & 0.680 & \textbf{0.924} \\
        & ATV $\downarrow$ & -- & -- & 2.540 & \underline{1.347} & 1.502 & -- & -- & \textbf{1.187} \\
        & MSE $\downarrow$ & -- & -- & -- & -- & -- & \underline{0.015} & 0.357 & \textbf{0.014} \\
\addlinespace
$\rho=0.6, B=0.75$ & ACC $\uparrow$   & 0.591 & 0.702 & 0.608 & \underline{0.909} & -- & 0.688 & 0.676 & \textbf{0.921} \\
        & ATV $\downarrow$ & -- & -- & 2.435 & \underline{1.270} & 1.547 & -- & -- & \textbf{1.176} \\
        & MSE $\downarrow$ & -- & -- & -- & -- & -- & \underline{0.014} & 0.356 & \textbf{0.013} \\
\midrule
$\rho=0.2, B=1.25$ & ACC $\uparrow$   & 0.553 & 0.579 & 0.566 & \underline{0.880} & -- & 0.586 & 0.663 & \textbf{0.910} \\
        & ATV $\downarrow$ & -- & -- & 3.544 & \underline{1.662} & 1.965 & -- & -- & \textbf{1.177} \\
        & MSE $\downarrow$ & -- & -- & -- & -- & -- & \underline{0.016} & 0.359 & \textbf{0.015} \\
\addlinespace
$\rho=0.6, B=1.25$ & ACC $\uparrow$   & 0.652 & 0.565 & 0.570 & \underline{0.853} & -- & 0.572 & 0.661 & \textbf{0.890} \\
        & ATV $\downarrow$ & -- & -- & 3.628 & \underline{1.611} & 1.719 & -- & -- & \textbf{1.166} \\
        & MSE $\downarrow$ & -- & -- & -- & -- & -- & \underline{0.015} & 0.360 & \textbf{0.014} \\
\bottomrule
\end{tabular}%
}
\caption{Clustering (ACC), alignment (ATV), and reconstruction (MSE) on trivariate synthetic data over \(B\) and \(\rho\); means over 50 repetitions. Bold/underline mark the best/second-best distinct values per row, and ``--'' denotes not applicable.}
\label{tab:sim-comparison}
\end{table}

\paragraph{Recovery of amplitude and phase.}
The simulation provides ground-truth amplitude functions $\{\mathbf{f}_i\}$ and
inverse warps $\{\omega_i(t)\}$, which are withheld from modeling. We assess
recovery by comparing the AP-FAE estimates $\{\widehat{\mathbf{f}}_i\}$ and
$\{\widehat{\omega}_i(t)\}$ against these truths. Figure~\ref{fig:sim-d3-interpretability}(a)
compares the first two multivariate functional principal components (FPCs) of
$\{\mathbf{f}_i\}$ (black solid) and $\{\widehat{\mathbf{f}}_i\}$ (blue dashed);
Figure~\ref{fig:sim-d3-interpretability}(b) compares a subset of true
$\{\omega_i(t)\}$ (dashed) and estimated $\{\widehat{\omega}_i(t)\}$ (solid).
The close agreement demonstrates effective amplitude and phase recovery. Full
reconstruction is shown in Figure~\ref{fig:sim-d3-workflow} in
Appendix~\ref{app:simu-visu}.

\begin{figure}[!htbp]
\centering
\begin{subfigure}[t]{.64\linewidth}
  \centering
  \vspace{0pt}
  \includegraphics[width=\linewidth]{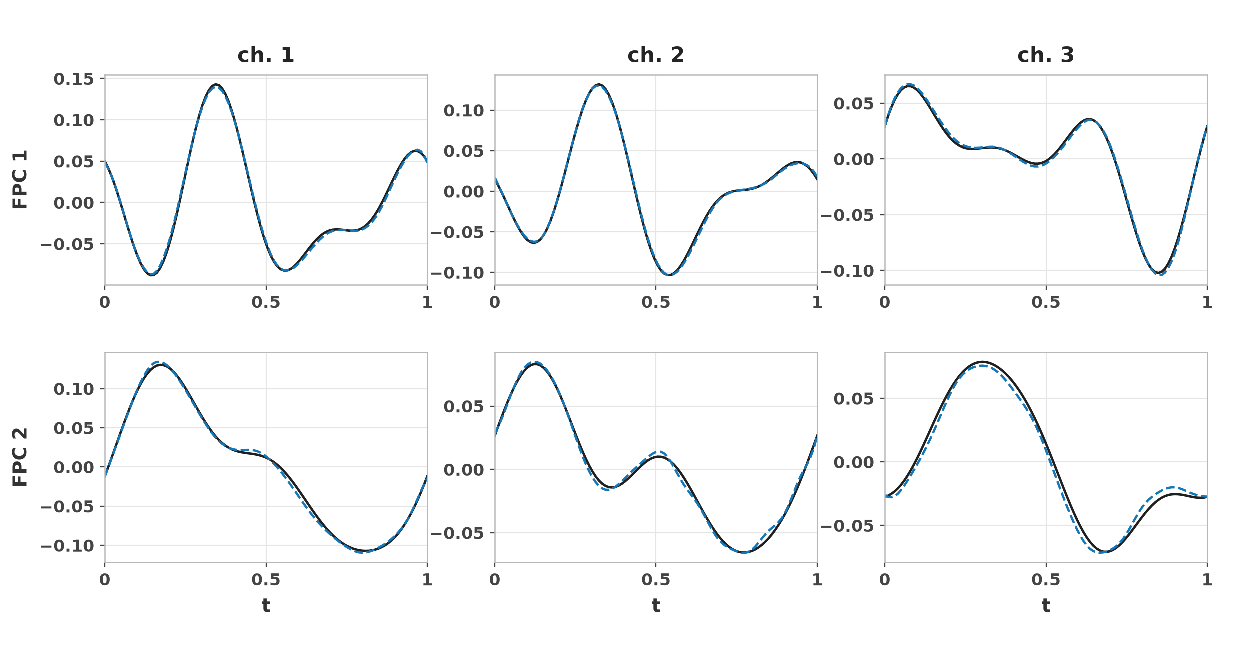}
  \caption{}
\end{subfigure}\hfill
\begin{subfigure}[t]{.32\linewidth}
  \centering
  \vspace{0pt}
  \includegraphics[width=\linewidth]{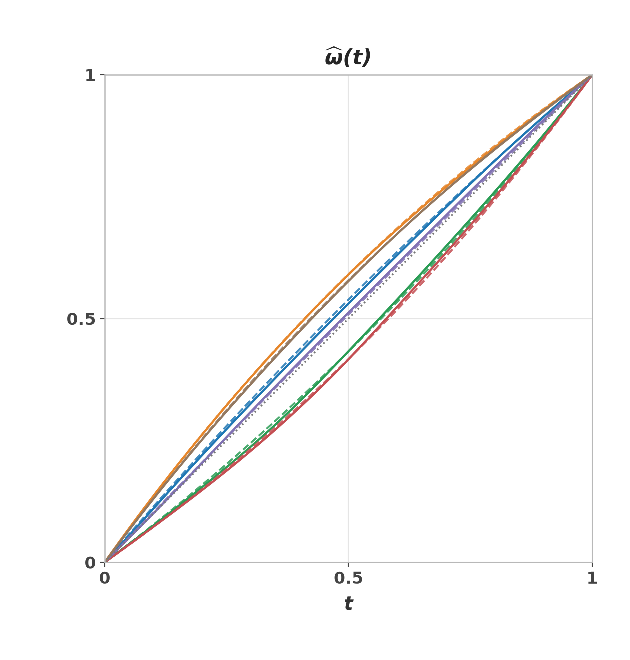}
  \caption{}
\end{subfigure}
\caption{Amplitude and phase recovery (\(\rho=0.2\), \(B=0.75\)). (a) True (black solid) vs.\ estimated (blue dashed) multivariate functional principal components (FPCs; first two shown) for channels 1--3. (b) True (dashed) vs.\ estimated (solid) inverse warps for six randomly chosen subjects (distinct colors); the black dotted line is the identity.}
\label{fig:sim-d3-interpretability}
\end{figure}

\paragraph{Numerical validation of amplitude error bound.}
We numerically validate \Cref{thm:amplitude-recovery-noise} on 50 AP-FAE fits per simulation setting. All \(200\) fits satisfy the bound, with
amplitude errors \(0.214\)--\(0.253\) and bounds \(0.434\)--\(0.523\),
confirming the theory while showing the bound is conservative. Full details are
in Appendix~\ref{app:bound-validation}.

\subsection{Real-World Data}
\label{sec:real-world}

\paragraph{Datasets.}
We evaluate on six benchmark datasets: three multivariate from the \href{https://www.timeseriesclassification.com/index.php}{UEA archive}
and three univariate from the \href{https://www.timeseriesclassification.com/index.php}{UCR archive}
\citep{Bagnall2018UEA,UCRArchive2018}. \textbf{BasicMotions} ($D=6$, 4 classes) contains
80 six-channel smartwatch recordings of walking, resting, running, and
badminton, each with 100 time points. \textbf{ERing} ($D=4$, 6 classes) comprises 300
four-channel wearable-ring samples of six finger postures, each with 65 time
points. \textbf{UWaveGestureLibrary} ($D=3$, 8 classes) includes 4,479 three-axis
accelerometer recordings of eight hand gestures, each with 315 time points. \textbf{CBF}
($D=1$, 3 classes) consists of 930 univariate plateau, rising-ramp, and
falling-ramp signals with random onset and duration, each of length 128. \textbf{Plane}
($D=1$, 7 classes) contains 210 univariate contour sequences from top-view
aircraft images, each of length 144. \textbf{Symbols} ($D=1$, 6 classes) comprises 1,020
univariate recordings of horizontal motion as participants draw displayed
symbols, each with 398 observations.

\paragraph{Performance comparison.}
Table~\ref{tab:real-comparison} shows that AP-FAE achieves the best ACC on five
of six datasets, the best ATV on five of six, and the best MSE on all six; the
exceptions are \textbf{UWaveGestureLibrary} (SRC best ACC, \(0.888\) vs.\ \(0.811\)) and
\textbf{ERing} (SRC best ATV, \(4.773\) vs.\ \(5.554\)). The functional autoencoders FAE and FAEclust, which lack a
phase pathway, are the runners-up in MSE on every dataset but are never best in
clustering, supporting the benefit of explicit amplitude--phase separation.
Alignment and reconstruction quality do not fully determine clustering: SRC
attains the best ATV on \textbf{ERing} yet is second in ACC, and AP-FAE attains the best
ATV on \textbf{UWaveGestureLibrary} yet is second in ACC, indicating that the three
criteria capture complementary aspects of the fit.
Figure~\ref{fig:realdata-uwave-workflow} visualizes amplitude--phase
disentanglement and reconstruction for one AP-FAE fit on channel 1 of
\textbf{UWaveGestureLibrary}; additional visualizations are in
Appendix~\ref{app:real-visualizations}.

\begin{table}[!htbp]
\centering
\small
\setlength{\tabcolsep}{3.2pt}
\renewcommand{\arraystretch}{1.05}
\resizebox{\linewidth}{!}{%
\begin{tabular}{lccccccccc}
\toprule
Datasets&Metric & Funclust & funHDDC & JPCCA & SRC & SrvfRegNet & FAE & FAEclust & AP-FAE \\
\midrule
\textbf{BasicMotions} &ACC $\uparrow$   & 0.408 & 0.439 & 0.250 & 0.378 & -- & \underline{0.506} & 0.513 & \textbf{0.807} \\
($D=6$, 4 classes)&ATV $\downarrow$ & -- & -- & \underline{107.025} & 113.544 & 114.108 & -- & -- & \textbf{61.116} \\
&MSE $\downarrow$ & -- & -- & -- & -- & -- & \underline{17.955} & 19.252 & \textbf{15.452} \\
\midrule
\textbf{ERing}& ACC $\uparrow$   & 0.219 & 0.745 & 0.493 & \underline{0.903} & -- & 0.862 & 0.488 & \textbf{0.943} \\
($D=4$, 6 classes)&ATV $\downarrow$ & -- & -- & 10.317 & \textbf{4.773} & 12.826 & -- & -- & \underline{5.554} \\
&MSE $\downarrow$ & -- & -- & -- & -- & -- & \underline{0.257} & 0.620 & \textbf{0.231} \\
\midrule
\textbf{UWaveGestureLibrary}&ACC $\uparrow$   & 0.211 & 0.711 & 0.184 & \textbf{0.888} & -- & 0.707 & 0.538 & \underline{0.811} \\
($D=3$, 8 classes)&ATV $\downarrow$ & -- & -- & 20.478 & \underline{14.771} & 37.213 & -- & -- & \textbf{13.209} \\
&MSE $\downarrow$ & -- & -- & -- & -- & -- & \underline{0.327} & 0.562 & \textbf{0.275} \\
\midrule
\textbf{CBF} &ACC $\uparrow$   & 0.345 & 0.542 & \underline{0.870} & 0.817 & -- & 0.596 & 0.335 & \textbf{0.911} \\
($D=1$, 3 classes)&ATV $\downarrow$ & -- & -- & 145.374 & \underline{6.153} & 16.712 & -- & -- & \textbf{4.515} \\
&MSE $\downarrow$ & -- & -- & -- & -- & -- & \underline{0.254} & 0.402 & \textbf{0.213} \\
\midrule
\textbf{Plane} &ACC $\uparrow$   & 0.192 & 0.810 & 0.267 & \underline{0.868} & -- & 0.804 & 0.734 & \textbf{0.875} \\
($D=1$, 7 classes)&ATV $\downarrow$ & -- & -- & \underline{0.292} & 0.297 & 0.920 & -- & -- & \textbf{0.057} \\
&MSE $\downarrow$ & -- & -- & -- & -- & -- & \underline{0.075} & 0.418 & \textbf{0.039} \\
\midrule
\textbf{Symbols}&ACC $\uparrow$   & 0.193 & 0.564 & 0.178 & \underline{0.738} & -- & 0.677 & 0.626 & \textbf{0.765} \\
($D=1$, 6 classes)&ATV $\downarrow$ & -- & -- & \underline{3.370} & 3.915 & 4.091 & -- & -- & \textbf{1.781} \\
&MSE $\downarrow$ & -- & -- & -- & -- & -- & \underline{0.059} & 0.566 & \textbf{0.038} \\
\bottomrule
\end{tabular}%
}
\caption{Clustering (ACC), alignment (ATV), and reconstruction (MSE) on real datasets; \(D\) denotes the number of channels. Means over 50 repetitions; bold/underline mark the best and second-best values per row, and ``--'' denotes not applicable.}
\label{tab:real-comparison}
\end{table}

\begin{figure}[!htbp]
\centering
\includegraphics[width=\linewidth]{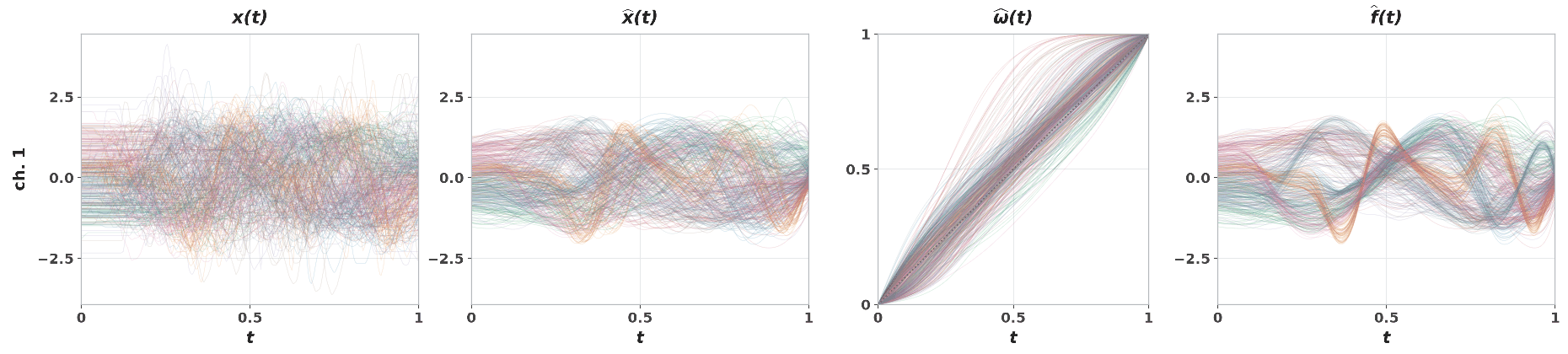}
\caption{First channel of an AP-FAE fit on \textbf{UWaveGestureLibrary} (\(D=3\)). Columns show raw \(\mathbf{x}(t)\), reconstructed \(\widehat{\mathbf{x}}(t)\), estimated inverse warps \(\widehat{\omega}(t)\), and estimated amplitudes \(\widehat{\mathbf{f}}(t)\); colors denote classes. All three channels appear in Figure~\ref{fig:uwave} in Appendix~\ref{app:real-visualizations}.}
\label{fig:realdata-uwave-workflow}
\end{figure}

\paragraph{Ablation on embedding usage and size.}
Clustering is performed on the amplitude embedding \(\mathbf{a}\), since the
goal is to group curves by shape under temporal misalignment. We also test
clustering on the phase embedding \(\mathbf{p}\) alone and on both embeddings
\([\mathbf{a},\mathbf{p}]\). As shown in
Table~\ref{tab:amplitude-phase-clustering}, clustering on \(\mathbf{a}\) alone
performs best, while adding \(\mathbf{p}\) degrades performance because phase
variation confuses the amplitude pattern; clustering on \(\mathbf{p}\) alone can
still be informative, and phase variation may carry useful signal for other tasks
such as regression. These results highlight the importance of disentangling
amplitude and phase. We further study the sensitivity to the embedding sizes
\(r\) and \(q\). Figure~\ref{fig:latent-allocation-acc} reports clustering
performance for \((r,q)=(2,6),(3,5),(4,4),(5,3),(6,2)\) across all datasets,
where \(r=q=4\) is consistently competitive and is used throughout.
 
\begin{figure}[!htbp]
\centering
\begin{minipage}[b]{0.48\linewidth}
  \centering
  \small
  \setlength{\tabcolsep}{4pt}
  \renewcommand{\arraystretch}{1.05}
  \resizebox{0.95\linewidth}{!}{%
  \begin{tabular}{lccc}
  \toprule
  Dataset & $\mathbf{a}$ & $\mathbf{p}$ & $[\mathbf{a},\mathbf{p}]$ \\
  \midrule
  BasicMotions ($D=6$) & \textbf{0.807} & \underline{0.800} & 0.695 \\
  ERing ($D=4$) & \textbf{0.943} & 0.721 & \underline{0.840} \\
  UWaveGestureLibrary ($D=3$) & \textbf{0.811} & 0.536 & \underline{0.701} \\
  \midrule
  CBF ($D=1$) & \textbf{0.911} & 0.585 & \underline{0.821} \\
  Plane ($D=1$) & \textbf{0.875} & 0.739 & \underline{0.854} \\
  Symbols ($D=1$) & \textbf{0.765} & 0.683 & \underline{0.753} \\
  \bottomrule
  \end{tabular}}
\captionof{table}{Clustering accuracy (ACC) using amplitude \(\mathbf{a}\), phase \(\mathbf{p}\), or both embeddings.}
  \label{tab:amplitude-phase-clustering}
\end{minipage}
\begin{minipage}[b]{0.48\linewidth}
  \centering
  \includegraphics[width=0.95\linewidth]{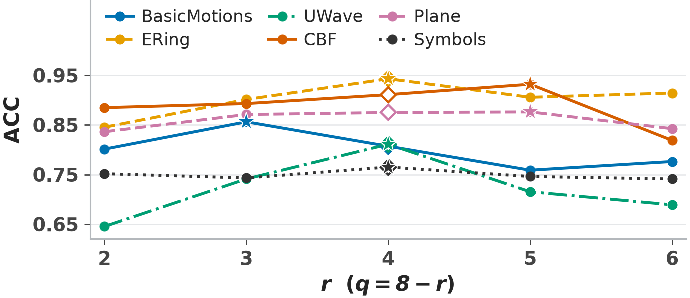}
\captionof{figure}{Sensitivity to embedding sizes; stars mark the best.}
  \label{fig:latent-allocation-acc}
\end{minipage}\hfill
\end{figure}

\section{Discussion}
\label{sec:discussion}

\paragraph{Robustness to irregular sampling and missing observations.}
We tested random missingness and uneven sampling in the trivariate synthetic
setting ($B=0.75$, $\rho=0.2$) over 50 repetitions. Clustering ACC remained
close to the complete-data baseline across all conditions, but clustering
stability did not imply accurate reconstruction in sparsely observed regions.
Under 80\% random missingness, remaining observations and the MSE
and ATV were $0.0072$ and $1.212$, respectively. In contrast, strongly uneven
sampling retained all observations but increased the MSE and ATV
to $0.0302$ and $1.275$. Thus, for alignment and reconstruction, 
what mattered was not merely how many observations were available, but whether 
they adequately covered the entire time interval. Details
and the full table are in Appendix~\ref{app:robustness_experiments}.

\paragraph{Extreme warping or no warping.}
We tested continuous warps with short plateaus and nondecreasing warps with
upward jumps. Mean ACC was nearly unchanged relative to the strictly increasing
baseline, while longer plateaus and larger jumps increased ATV and MSE. Since
AP-FAE assumes continuous, strictly positive-velocity warps, it cannot exactly
represent plateaus or jumps and instead approximates them with small positive
velocities and steep transitions. Amplitude adjustments can still preserve
clustering and low average MSE, so these metrics do not imply recovery of the
true warp. Details are in Appendix~\ref{app:robustness_experiments}. For the
no-warping case (\(\omega(t)=t\), \(B=0\)), AP-FAE performs comparably to FAE (Table \ref{tab:simulation-b0} in Appendix \ref{app:simulation-b0}) and
estimates \(\widehat{\omega}(t)\approx t\) (Figure \ref{fig:simulation-b0-rho060} in Appendix~\ref{app:simu-visu}), confirming applicability whether or
not misalignment exists. 

\paragraph{Scalability.}
AP-FAE behaves stably at both ends of the sample-size and resolution ranges we tested. Increasing the number of subjects from $90$ to $30{,}000$ lowers reconstruction error monotonically (MSE $0.016\!\to\!0.011$) while clustering accuracy and alignment remain essentially flat; training time grows approximately linearly in $N$. Varying the number of grid points from $30$ to $1{,}000$ leaves MSE and ACC nearly unchanged, and training time changes by less than $7\%$ over a $33$-fold increase in $M$. The apparent growth in raw ATV with $M$ is an artifact of its approximately $\sqrt{M}$ scaling: after normalization, ATV lies within $0.116$--$0.118$ throughout. Full protocols, timing setup, and per-condition results appear in Appendix~\ref{app:scalability}.

\section{Conclusion}
\label{sec:conclusion}

We proposed AP-FAE, an unsupervised framework for temporally misaligned
functional data that jointly learns nonlinear amplitude and phase
representations, applying to both univariate and multivariate settings with a
primary focus on the latter. By encoding all channels into separate amplitude and phase
embeddings, decoding channel-specific amplitude functions in canonical time, and
modeling phase variation with a shared monotone, endpoint-preserving warp,
AP-FAE disentangles shape from timing under a reconstruction-only objective.
This yields compact, interpretable embeddings that are not tied to a downstream
task and support flexible use such as amplitude-based clustering. We also
established a theoretical bound linking amplitude recovery to registration,
reconstruction, and noise errors, and validated it numerically. Across synthetic
and six real-world benchmarks, AP-FAE achieved strong clustering, alignment, and
reconstruction, outperforming state-of-the-art baselines on most metrics;
notably, clustering on amplitude embeddings alone consistently surpassed
clustering on joint amplitude--phase embeddings, confirming the benefit of
explicit disentanglement.

Our framework has limitations. It assumes a continuous, positive-velocity warp
shared across channels, and its performance degrades when the true warps contain
long constant segments or upward jumps, or when observations have severe
missingness or strongly uneven sampling; in such cases, stable
clustering and low average reconstruction error should not be interpreted as
recovery of the true warp. In addition, the amplitude and phase embedding sizes
are currently fixed rather than learned from data. Future work includes
data-adaptive embedding sizes, richer warp families that accommodate
discontinuities and subject-specific grids, and downstream tasks such as
functional regression using amplitude embeddings, phase embeddings, or both.

\newpage
\subsection*{AI use statement}

In this work, we used generative ChatGPT and Deepseek for polishing the writing
of the manuscript and for assisting with code implementation. We have not used
generative AI tools for research ideation, method design, or the analysis and
interpretation of experimental results, and the remaining required disclosure
tasks are not applicable to this work. We have reviewed all AI-assisted work.
Specifically, AI-assisted text edits were limited to improving grammar and
clarity, and every edit was checked by the authors to ensure that the technical
content and claims were unchanged. All AI-assisted code was manually reviewed,
tested, and verified for correctness by the authors before being used in any
experiment. We take responsibility for the final content of this work,
including text, claims or artifacts produced with the aid of generative AI.

\subsection*{Reproducibility statement}

We have made efforts to ensure the reproducibility of our results. The full
source code of AP-FAE, including training and evaluation scripts, is available
at the anonymous repository linked in the abstract. The method is described in
Section~\ref{sec:model}, with algorithms and implementation details in
Appendix~\ref{app:algorithms} and Appendix~\ref{app:implementation}. Complete
proofs and assumptions for the theoretical results in Section~\ref{sec:insight}
are given in Appendix~\ref{app:amplitude-recovery}. The experimental setup, including
datasets, baselines, and evaluation metrics, is given in
Section~\ref{sec:real-world} and Appendix~\ref{app:evaluation-metrics}. Model sizes and baseline settings are listed in
Appendix~\ref{app:simulation-protocol}, the synthetic data generation procedure is in
Appendix~\ref{app:simulation-design}.

\FloatBarrier
\bibliographystyle{plainnat}
\bibliography{references}

\clearpage
\appendix
\renewcommand{\theequation}{A.\arabic{equation}}
\setcounter{equation}{0}
\renewcommand{\thetable}{A\arabic{table}}
\setcounter{table}{0}
\renewcommand{\thefigure}{A\arabic{figure}}
\setcounter{figure}{0}

\section{Details of Algorithms and Complexity}
\label{app:algorithms-complexity}

\subsection{Algorithms}
\label{app:algorithms}

\begin{algorithm}[!htbp]
\caption{Time Warping}
\label{alg:warp}
\begin{algorithmic}[1]
\Require Phase embedding $\mathbf{p}_i$; network parameters $\psi$;
         grid $0=t_1<\cdots<t_M=1$
\Ensure Inverse warp $\{\widehat{\omega}_i(t_m)\}_{m=1}^M$
\State Compute network outputs
       $\rho_{im}\gets \mathcal{G}_\psi(t_m,\mathbf{p}_i)$
       for $m=1,\ldots,M$
\State Compute stabilized positive densities
       $u_{im}\gets\exp\{\rho_{im}-\max_\ell\rho_{i\ell}\}$
       for $m=1,\ldots,M$
\State Compute trapezoidal increments
       $\displaystyle
       \Delta_{im}\gets
       \frac{u_{im}+u_{i,m+1}}{2}(t_{m+1}-t_m)$
       for $m=1,\ldots,M-1$
\State Accumulate integrals
       $\displaystyle H_{i1}\gets 0,\qquad
       H_{im}\gets\sum_{\ell=1}^{m-1}\Delta_{i\ell}$
       for $m=2,\ldots,M$
\State Normalize
       $\widehat{\omega}_i(t_m)\gets H_{im}/H_{iM}$
       for $m=1,\ldots,M$
\State \Return inverse warp
       $\{\widehat{\omega}_i(t_m)\}_{m=1}^M$
\end{algorithmic}
\end{algorithm}

\begin{algorithm}[!htbp]
\caption{Training AP-FAE}
\label{alg:twfae}
\begin{algorithmic}[1]
\Require Observations \(\{x_{id}(t_m)\}_{i=1,d=1,m=1}^{N,D,M}\); latent dimensions \(r,q\); number of B-spline bases \(J\)
\Ensure Fitted model, reconstructions and embeddings
\State Initialize encoder parameters \(\theta\), amplitude-decoder parameters \(\phi\), and phase-decoder parameters \(\psi\)
\State Precompute B-spline basis values
\While{stopping criterion is not satisfied}
    \State \((\mathbf{a}_i,\mathbf{p}_i)\gets \mathcal{E}_\theta(\mathbf{x}_i)\)
           for \(i=1,\ldots,N\)
    \State \(\widehat f_{id}\gets \mathcal{D}_\phi(\mathbf{a}_i)\)
           for \(i=1,\ldots,N\), \(d=1,\ldots,D\)
    \State \(\widehat\omega_i\gets\Call{MonotoneWarp}{\mathbf{p}_i,\psi}\)
           using Algorithm~\ref{alg:warp}
    \State \(\widehat x_{id}(t_m)\gets
           \Call{LinearInterp}{\widehat f_{id},\widehat\omega_i(t_m)}\)
    \State \(\displaystyle
       \mathcal{L}_{\mathrm{rec}}\gets
       \frac{1}{NDM}\sum_{i=1}^N\sum_{d=1}^D\sum_{m=1}^M
       [x_{id}(t_m)-\widehat x_{id}(t_m)]^2\)
    \State Update \(\theta,\phi,\psi\) using AdamW and \(\nabla\mathcal{L}_{\mathrm{rec}}\)
\EndWhile
\State \Return embeddings \(\{\mathbf{a}_i,\mathbf{p}_i\}\) and functions
       \(\{\widehat f_{id},\widehat\omega_i\}\)
\end{algorithmic}
\end{algorithm}

\subsection{Phase Decoder Implementation}
\label{app:implementation}
We discretize Eq.~\eqref{eq:normalized-integrated-exponential}
on the common grid $0=t_1<\cdots<t_M=1$ using cumulative trapezoidal integration.
A two-hidden-layer ReLU MLP of width $L$ produces
$\rho_{im}=G_\psi(t_m,p_i)$.
For numerical stability, we set
$u_{im}=\exp\{\rho_{im}-\max_{\ell}\rho_{i\ell}\}$
and compute
\[
\Delta_{im}
=
\frac{u_{im}+u_{i,m+1}}{2}(t_{m+1}-t_m),
\qquad m=1,\ldots,M-1.
\]
The estimated inverse warp is
\begin{equation}
\widehat{\omega}_i(t_1)=0,
\qquad
\widehat{\omega}_i(t_m)
=
\frac{\sum_{\ell=1}^{m-1}\Delta_{i\ell}}
     {\sum_{\ell=1}^{M-1}\Delta_{i\ell}},
\qquad m=2,\ldots,M.
\label{eq:discrete_inverse_warp}
\end{equation}
Positive increments ensure strict monotonicity, and normalization
gives $\widehat{\omega}_i(1)=1$.
Subtracting the maximum stabilizes exponentiation without changing
the normalized warp.

At each $\widehat{\omega}_i(t_m)$, we linearly interpolate the two
neighboring values of $\widehat{f}_{id}$ on the amplitude-decoder grid.
The interpolation locations and weights are shared across channels,
and gradients pass through both the amplitude values and the warp
locations.

Our experiments use AdamW with learning rate $10^{-3}$ and weight
decay $10^{-4}$. Model sizes and baseline settings are listed in
Section~C.2. For downstream clustering, the 20 K-means initializations
are independent clustering restarts and do not affect neural
initialization.

\subsection{Computational and Space Complexity}
\label{app:complexity}

Let \(N\) be the number of samples, \(D\) the number of functional channels,
\(M\) the number of grid points, and \(J\) the number of B-spline bases.
Let \(K\) be the number of functional encoder features, \(S\) the minibatch size,
\(r\) and \(q\) the amplitude and phase latent dimensions, and \(L\) the hidden width
of the warping network.

\paragraph{Time complexity.}
For dense grid-based encoding and coefficient-based amplitude decoding, one epoch costs
\[
\mathcal{O}\!\left(
\left\lceil \frac{N}{S}\right\rceil K D J M
+
N\Bigl[
DK(M+J)+DJM+DM
+K(2r+q)
+M(Lq+L^2)
\Bigr]
\right).
\]
The first term evaluates the learned encoder weight functions on the grid after each
parameter update. The remaining terms account for functional encoding, amplitude
decoding, time-warping evaluation, interpolation, and reconstruction loss.
Backpropagation has the same asymptotic order. For fixed grid resolution, basis size,
network widths, and minibatch size, the cost scales linearly in both \(N\) and \(D\).
The shared warping network avoids an additional factor of \(D\) in the phase pathway.

\paragraph{Space complexity.}
Writing \(P\) for the number of trainable parameters, minibatch training requires
\[
\mathcal{O}\!\left(
P+MJ+DKM
+
S\bigl[DM+DJ+ML+K+r+q\bigr]
\right)
\]
working memory. This includes model and optimizer storage, basis evaluations,
encoder weight functions, and intermediate activations. Storing the full dataset
additionally requires \(\mathcal{O}(NDM)\) memory. The shared warp needs only
\(\mathcal{O}(SM)\) interpolation locations and weights, while channel-specific
inputs and reconstructions require \(\mathcal{O}(SDM)\).

\section{Details on the Theory}
\label{app:amplitude-recovery}

\begin{proof}[Proof of Theorem~\ref{thm:amplitude-recovery-noise}]
Recall that $\mathbf{x}_i=\mathbf{f}_i\circ\omega_i+\boldsymbol e_i$
is the observed curve and
$\widehat{\mathbf{x}}_i=\widehat{\mathbf{f}}_i\circ\widehat{\omega}_i$
its reconstruction. Here $\mathbf{f}_i$ is the true amplitude curve,
$\omega_i$ is the true inverse warp, and $\boldsymbol\nu_i$ is measurement
noise. The quantities in the theorem are
\[
\epsilon_g=\max_i\|\widehat{\omega}_i-\omega_i\|_\infty,
\qquad
\epsilon_x=\left\{\frac1N\sum_{i=1}^N
\|\widehat{\mathbf{x}}_i-\mathbf{x}_i\|_{L^2}^2\right\}^{1/2},
\qquad
\tau=\left\{\frac1N\sum_{i=1}^N
\|\boldsymbol e_i\|_{L^2}^2\right\}^{1/2}.
\]
Thus $\epsilon_g$ measures the largest inverse-warp error, while
$\epsilon_x$ and $\tau$ are the root mean squared $L^2$ reconstruction
error and noise level across subjects. The reconstruction error is measured
against the noisy observations. The constants $L$ and $C$ are uniform
Lipschitz bounds for the amplitude curves $\mathbf{f}_i$ and fitted
inverse warps $\widehat{\omega}_i$, respectively. All $L^2$ norms below
are taken over $[0,1]$ and sum over the $D$ channels, as in the theorem.

Write $R_i=\widehat{\mathbf{x}}_i-\mathbf{x}_i$ for the reconstruction
residual. Since $\widehat{\omega}_i$ is a strictly increasing continuous
bijection of $[0,1]$, substitution of
$t=\widehat{\omega}_i^{-1}(s)$ gives
\[
\begin{aligned}
\widehat{\mathbf{f}}_i(s)-\mathbf{f}_i(s)
&=\widehat{\mathbf{x}}_i\{\widehat{\omega}_i^{-1}(s)\}-\mathbf{f}_i(s)\\
&=R_i\{\widehat{\omega}_i^{-1}(s)\}
 +\boldsymbol\nu_i\{\widehat{\omega}_i^{-1}(s)\}\\
&\quad+\mathbf{f}_i\{\omega_i(\widehat{\omega}_i^{-1}(s))\}
 -\mathbf{f}_i(s).
\end{aligned}
\]
The last term is bounded using the Lipschitz condition on $\mathbf{f}_i$.
Indeed, $s=\widehat{\omega}_i(t)$ ranges over $[0,1]$ as $t$ does, so
\[
\sup_{s\in[0,1]}
\left|\omega_i\{\widehat{\omega}_i^{-1}(s)\}-s\right|
=\sup_{t\in[0,1]}|\omega_i(t)-\widehat{\omega}_i(t)|
\leq\epsilon_g.
\]
Consequently,
\[
\begin{aligned}
\|\mathbf{f}_i\circ\omega_i\circ\widehat{\omega}_i^{-1}
  -\mathbf{f}_i\|_{L^2}^2
&=\int_0^1\|\mathbf{f}_i\{\omega_i(\widehat{\omega}_i^{-1}(s))\}
             -\mathbf{f}_i(s)\|_2^2\,ds\\
&\leq L^2\int_0^1
 \left|\omega_i\{\widehat{\omega}_i^{-1}(s)\}-s\right|^2\,ds
\leq L^2\epsilon_g^2.
\end{aligned}
\]

For the residual and noise terms, the $C$-Lipschitz condition implies
that $\widehat{\omega}_i$ is absolutely continuous, with
$0\leq\widehat{\omega}_i'(t)\leq C$ almost everywhere.
Changing variables by $s=\widehat{\omega}_i(t)$ yields, for any
$\mathbf{u}\in L^2([0,1];\mathbb{R}^D)$,
\[
\begin{aligned}
\|\mathbf{u}\circ\widehat{\omega}_i^{-1}\|_{L^2}^2
&=\int_0^1\|\mathbf{u}\{\widehat{\omega}_i^{-1}(s)\}\|_2^2\,ds\\
&=\int_0^1\|\mathbf{u}(t)\|_2^2\widehat{\omega}_i'(t)\,dt
\leq C\|\mathbf{u}\|_{L^2}^2.
\end{aligned}
\]
Applying this estimate to $R_i$ and $\boldsymbol\nu_i$ in the
decomposition above, and then using the triangle inequality, gives
\[
\begin{aligned}
\|\widehat{\mathbf{f}}_i-\mathbf{f}_i\|_{L^2}
&\leq\|R_i\circ\widehat{\omega}_i^{-1}\|_{L^2}
 +\|\boldsymbol e_i\circ\widehat{\omega}_i^{-1}\|_{L^2}
 +L\epsilon_g\\
&\leq\sqrt C\,\|R_i\|_{L^2}
 +\sqrt C\,\|\boldsymbol e_i\|_{L^2}+L\epsilon_g.
\end{aligned}
\]
Taking the root mean square over subjects and applying Minkowski's
inequality in $\mathbb{R}^N$, we obtain
\[
\begin{aligned}
\left\{\frac1N\sum_{i=1}^N
 \|\widehat{\mathbf{f}}_i-\mathbf{f}_i\|_{L^2}^2\right\}^{1/2}
&\leq\left\{\frac1N\sum_{i=1}^N
 \left(\sqrt C\,\|R_i\|_{L^2}
 +\sqrt C\,\|\boldsymbol e_i\|_{L^2}+L\epsilon_g\right)^2
 \right\}^{1/2}\\
&\leq\sqrt C\left\{\frac1N\sum_{i=1}^N
 \|R_i\|_{L^2}^2\right\}^{1/2}
 +\sqrt C\left\{\frac1N\sum_{i=1}^N
 \|\boldsymbol e_i\|_{L^2}^2\right\}^{1/2}
 +L\epsilon_g\\
&=L\epsilon_g+\sqrt C\,(\epsilon_x+\tau).
\end{aligned}
\]
This proves the stated bound. The argument holds for each realization
of the noise and requires no distributional assumptions on
$\boldsymbol e_i$.
\end{proof}

\section{Details of Experiments}
\label{app:experimental-settings}

\subsection{Evaluation Metrics}
\label{app:evaluation-metrics}

Let \(N\), \(M\), and \(D\) denote the sample size, number of grid points, and
number of channels.

\paragraph{Clustering.}
Let \(N_{k\ell}=\sum_i\mathbf 1\{y_i=k,\widehat y_i=\ell\}\),
\(n_k=\sum_\ell N_{k\ell}\), and \(m_\ell=\sum_k N_{k\ell}\).
ACC uses optimal one-to-one label matching \(\mathcal P\):
\begin{equation}
  \operatorname{ACC}
  =\frac{1}{N}\max_{\mathcal P}\sum_{(k,\ell)\in\mathcal P}N_{k\ell}.
  \label{eq:app-metric-acc}
\end{equation}

\paragraph{Registration.}
For aligned observations \(A_i=[x_{id}\{\widehat \omega^{-1}_i(t_m)\}]_{m,d}\),
define \(\overline A_k=n_k^{-1}\sum_{i:y_i=k}A_i\) and
\(S_k=\sum_{i:y_i=k}\|A_i-\overline A_k\|_F^2\).
We use the pairwise ATV criterion of \citet{jiang2025deepfrc},
\begin{equation}
  \operatorname{ATV}
  =\frac{2}{K(K-1)}\sum_{k<\ell}
  \frac{S_k+S_\ell}
  {(n_k+n_\ell)\|\overline A_k-\overline A_\ell\|_F},
  \label{eq:app-metric-atv}
\end{equation}
where \(K\) is the number of true classes and the Frobenius norm covers all
grid points and channels. ATV uses class labels for evaluation but requires no
true warping functions.

\paragraph{Reconstruction.}
For \(\widehat x_{id}(t_m)=\widehat f_{id}\{\widehat \omega_i(t_m)\}\),
\begin{equation}
  \operatorname{MSE}
  =\frac{1}{NMD}\sum_{i=1}^N\sum_{m=1}^M\sum_{d=1}^D
  [x_{id}(t_m)-\widehat x_{id}(t_m)]^2.
  \label{eq:app-metric-mse}
\end{equation}

\subsection{Baseline Implementation Details}
\label{app:simulation-protocol}

\begin{itemize}[leftmargin=*,itemsep=3pt]
\item \textbf{FAE:} A functional autoencoder without phase variation, using the
reported functional encoder--decoder settings \citep{WuBeaulacCao2024}.
\item \textbf{FAEclust:} Official PyTorch implementation with elastic distance,
architecture \(16\)--\(8\)--\(4\)--\(8\)--\(16\)--\(16\)--\(16\)\citep{singh2025faeclust}.
\item \textbf{SrvfRegNet:} Local reproduction with a shared-warp multichannel
extension and Adam; provides alignment-only entries \citep{Chen2021}.
\item \textbf{SRC:} Scalar specialization for univariate examples and a
coordinate-list extension of the author-supplied likelihood for trivariate
curves \citep{zeng2019src}; both estimate a shared subject-specific warp without
class labels.
\item \textbf{JPCCA:} Spline-mixture reimplementation following the published
model \citep{gaffney2004jpcca}.
\item \textbf{Funclust and funHDDC:} Published R implementations
\citep{JACQUES2013164,BouveyronJacques2011funHDDC}.
\end{itemize}

\subsection{Synthetic Data Generation}
\label{app:simulation-design}

We adapt the basis-expansion construction of \citet{HsiehSunWangHonavar2021},
using Gaussian basis coefficients to produce individual curves around class
templates and adding a shared time warp. The main comparison contains three
classes of trivariate curves (\(D=3\)), with \(100\) subjects per class observed
at \(100\) equally spaced points on \([0,1]\).

\paragraph{Observed curves.}
For subject \(i\) and channel \(d\),
\begin{equation}
  x_{id}(t_j)=f_{id}\{\omega_i(t_j)\}+\epsilon_{idj},
  \qquad \epsilon_{idj}\overset{\mathrm{iid}}{\sim}N(0,0.10^2),
  \label{eq:sim-observed-d3}
\end{equation}
where \(f_{id}\) is the amplitude curve before warping, \(g_i^{-1}\) changes
its timing, and \(\epsilon_{idj}\) is independent measurement noise.

\paragraph{Time warping.}
Each subject has one warp shared by all channels:
\begin{equation}
  \omega_i(t)=
  \begin{cases}
    \{\exp(b_i t)-1\}/\{\exp(b_i)-1\}, & b_i\ne0,\\
    t, & b_i=0,
  \end{cases}
  \qquad b_i\sim\mathrm{Unif}(-B,B).
  \label{eq:sim-warp-d3}
\end{equation}
As in the model, \(\omega_i(t)\) maps observed time to aligned time. The map is
increasing and fixes both endpoints. We vary \(B\in\{0.75,1.25\}\) to control
the range of time distortions. Warps are independent of the amplitude
coefficients and class membership, so timing supplies no class information.

\paragraph{Amplitude curves and basis coefficients.}
The amplitude curve is a linear combination of three harmonics,
\begin{equation}
  f_{id}(t)=\sum_{h=1}^3 c_{idh}\phi_{dh}(t),
  \qquad d=1,\ldots,D.
  \label{eq:sim-amplitude-d3}
\end{equation}
For a subject in class \(k\), draw the coefficient vector across channels,
\(\boldsymbol c_{ih}=(c_{i1h},c_{i2h},c_{i3h})^\top\), from
\[
  \boldsymbol c_{ih}\sim
  N_3\!\left(\beta_{kh}\boldsymbol 1_3,
                  (0.20\beta_{kh})^2\Sigma_\rho\right),
  \qquad
  \Sigma_\rho=
  \begin{pmatrix}
    1&\rho&\rho\\
    \rho&1&\rho\\
    \rho&\rho&1
  \end{pmatrix},
\]
where \(\boldsymbol 1_3=(1,1,1)^\top\). Draws are independent across subjects
and harmonics. Each coefficient has class-specific mean \(\beta_{kh}\) and
standard deviation \(0.20|\beta_{kh}|\); coefficients of the same harmonic have
correlation \(\rho\) across channels. The three class mean vectors are fixed
across repetitions:
\begin{align*}
  \boldsymbol\beta_1&=(0.9175,\ 0.3195,\ -0.1380),\\
  \boldsymbol\beta_2&=(0.9175,\ -0.1205,\ 0.1260),\\
  \boldsymbol\beta_3&=(0.6150,\ 0.3910,\ -0.2480).
\end{align*}
We vary \(\rho\in\{0.2,0.6\}\) while keeping marginal coefficient variances
fixed. Here \(\rho\) describes coefficient correlation, rather than correlation
between observed curves. The channel-specific harmonics are
\begin{align*}
  \boldsymbol\phi_1(t)
    &= \{\sin(2\pi t),\ \sin(4\pi t),\ \cos(6\pi t)\},\\
  \boldsymbol\phi_2(t)
    &= \{0.75\sin(2\pi t+0.35),\ 0.95\sin(4\pi t+0.55),\
         0.80\cos(6\pi t+0.15)\},\\
  \boldsymbol\phi_3(t)
    &= \{0.90\sin(2\pi t-0.22),\ -0.65\sin(4\pi t+0.30),\
         0.50\cos(6\pi t-0.40)\}.
\end{align*}

\subsection{Numerical Validation of the Amplitude Recovery Bound}
\label{app:bound-validation}

We evaluate Eq.(\ref{eq:amplitude-recovery-bound-noise}) on 50  AP-FAE
fits for each of the four simulation settings. Amplitude and registration errors
are measured in the original generating time coordinate; reconstruction error is
measured against the noisy observations, and \(\tau\) is computed from the corresponding noise realizations. Table~\ref{tab:sim-amplitude-bound} compares the amplitude
root mean squared \(L^2\) error with
\(\epsilon=L\epsilon_g+\sqrt{C}\,(\epsilon_x+\tau)\). All \(200\) fits satisfy the bound.
Across settings, the mean amplitude error ranges from \(0.214\) to \(0.253\),
compared with mean bounds of \(0.434\)--\(0.523\), so the bound is conservative
in these simulations. These results provide direct numerical validation of
\Cref{thm:amplitude-recovery-noise} and confirm that the derived bound holds in
practice.

\begin{table}[!htbp]
\centering
\small
\setlength{\tabcolsep}{10pt}
\begin{tabular}{cccccc}
\toprule
\(B\) & \(\rho\) & \(L\) & \(C\) & Amplitude error & Upper bound \(\epsilon\) \\
\midrule
0.75 & 0.2 & 17.402 & 1.680 & 0.253 & 0.523 \\
0.75 & 0.6 & 18.006 & 1.746 & 0.252 & 0.523 \\
1.25 & 0.2 & 17.402 & 1.971 & 0.221 & 0.451 \\
1.25 & 0.6 & 18.006 & 2.020 & 0.214 & 0.434 \\
\bottomrule
\end{tabular}
\caption{Numerical assessment of the amplitude recovery bound for the four
trivariate simulation settings. Entries are means over 50 fits per
setting. Amplitude error denotes root mean squared \(L^2\) error.}
\label{tab:sim-amplitude-bound}
\end{table}



\subsection{Details of the Robustness Experiments}
\label{app:robustness_experiments}

\paragraph{Robustness to Irregular Sampling and Missing Observations.}
We evaluated robustness to random missingness and unevenly spaced observations in the trivariate synthetic setting with
\(B=0.75\) and \(\rho=0.2\). Each experiment was repeated 50 times.
Table~\ref{tab:irregular_sampling} reports clustering accuracy (ACC), adjusted
total variance (ATV), and reconstruction error (MSE) against the noise-free signal on
a dense reference grid.

\begin{table}[H]
\centering
\caption{Robustness to irregular sampling after retraining
(\(B=0.75\), \(\rho=0.2\)). Values are means over 50 repetitions.}
\label{tab:irregular_sampling}
\small
\begin{tabular}{lrrr}
\hline
Sampling condition & ACC \(\uparrow\) & ATV \(\downarrow\) & MSE \(\downarrow\) \\
\hline
Complete uniform grid & \(0.924\) & 1.187 & 0.0047 \\
Random missingness: 20\% & \(0.924\) & 1.177 & 0.0043 \\
Random missingness: 50\% & \(0.935\) & 1.181 & 0.0048 \\
Random missingness: 80\% & \(0.932\) & 1.212 & 0.0072 \\
Uneven sampling: \(\alpha_{\mathrm{samp}}=2\) & \(0.922\) & 1.185 & 0.0056 \\
Uneven sampling: \(\alpha_{\mathrm{samp}}=4\) & \(0.922\) & 1.275 & 0.0302 \\
\hline
\end{tabular}
\end{table}

\paragraph{Extreme Warping Cases.}
We further considered continuous warps containing short constant segments and
nondecreasing warps containing upward jumps. All curves were observed on the
complete grid. Each experiment was repeated 50 times.
Table~\ref{tab:extreme_warping} reports ACC, ATV, and reconstruction
MSE.

\begin{table}[H]
\centering
\caption{Sensitivity to extreme true warps after retraining
(\(B=0.75\), \(\rho=0.2\)). Values are means over 50 repetitions.}
\label{tab:extreme_warping}
\small
\begin{tabular}{lrrr}
\hline
Warp configuration & ACC \(\uparrow\) & ATV \(\downarrow\) & MSE \(\downarrow\) \\
\hline
Strictly increasing baseline & \(0.924\) & 1.187 & \(0.00470\) \\
One plateau: length 0.04 & \(0.933\) & 1.236 & \(0.00523\) \\
One plateau: length 0.08 & \(0.920\) & 1.346 & \(0.00756\) \\
One plateau: length 0.10 & \(0.929\) & 1.357 & \(0.00833\) \\
Two plateaus: length 0.04 each & \(0.915\) & 1.281 & \(0.00616\) \\
One upward jump: \(\Delta_{\mathrm{J}}=0.02\) & \(0.924\) & 1.211 & \(0.00506\) \\
One upward jump: \(\Delta_{\mathrm{J}}=0.05\) & \(0.927\) & 1.307 & \(0.00662\) \\
\hline
\end{tabular}
\end{table}

\paragraph{Experimental Setup.}
All robustness experiments used the same synthetic setting: 
three channels, three classes, 100 subjects per class, 
\(\rho=0.2\), \(B=0.75\), and 50 repetitions. 
AP-FAE was retrained from scratch under every condition. 
Training used1500 epochs, minibatch size 32, AdamW with learning rate \(10^{-3}\) 
and weight decay \(10^{-4}\), 20 basis functions, hidden width 16, 
and four-dimensional amplitude and phase codes. 
Clustering was performed once at the final checkpoint; ACC used true labels only for evaluation.

\paragraph{Irregular Sampling and Missing Observations.}
The two sampling mechanisms isolate different practical failures. 
Random missingness models scattered observation loss, 
such as intermittent sensor or recording failures. 
Uneven sampling models a valid but highly nonuniform design, 
such as measurements concentrated near the end of follow-up, 
and tests whether nominal sample size can mask poor temporal coverage.

Random missingness removed 20\%, 50\%, or 80\% of 
the coordinates. The two endpoints were always retained.

For uneven sampling, interior coordinates were selected 
without replacement from a 1001-point candidate grid 
and combined with the two endpoints. Selection weights were proportional to
\[
  t^{\alpha_{\mathrm{samp}}-1},
\]
where \(\alpha_{\mathrm{samp}}=2\) and \(4\) generated 
moderate and strong concentration toward the end of the domain.  
Within a replication, all subjects and channels shared the observation grid;

\paragraph{Warps with Constant Segments.}
A constant segment represents a temporary pause in latent progress: 
observed time advances while canonical time remains unchanged. 
It can also approximate saturation, rest, or repeated measurements of an unchanged latent state. 
This setting tests the strict positive-velocity assumption, 
because the true warp has zero derivative on a nonempty interval. 
The three single-plateau lengths vary the duration of one pause, 
while the two-plateau condition contrasts one contiguous pause with
 two separated pauses having the same total duration of 0.08.

For disjoint plateau intervals \([\ell_k,r_k]\), define
\[
h_{\mathrm{P}}(t)=
\frac{t-\sum_k\min\{\max(t-\ell_k,0),\,r_k-\ell_k\}}
     {1-L_{\mathrm{P}}},
\qquad
L_{\mathrm{P}}=\sum_k(r_k-\ell_k),
\]
and set
\[
  {\omega}_{\text{plateau}}(t)={\omega}\!\left(h_{\mathrm{P}}(t)\right).
\]
The resulting warp is continuous and nondecreasing, 
has constant segments, and satisfies \({\omega}_{\text{plateau}}(0)=0\) and \({\omega}_{\text{plateau}}(1)=1\). 
We tested one plateau of length 0.04, 0.08, or 0.10, 
and two plateaus of length 0.04 each. 
For each subject, the single-plateau centre was sampled from
a 0.01-spaced grid in \([0.30,0.70]\) and held fixed across the three lengths. 
The two-plateau centres were sampled from the corresponding grids in 
\([0.20,0.35]\) and \([0.65,0.80]\). 
Locations were generated independently of class, 
and all three channels shared the subject-specific warp.

\paragraph{Discontinuous Warps.}
We also tested warps with one upward jump. This represents a sudden advance in
latent time and violates the model's continuity assumption. We first modified
the time coordinate using
\[
h_{\mathrm{J}}(t)=(1-\Delta_{\mathrm{J}})t+
\Delta_{\mathrm{J}}\,\mathbf{1}_{\{t\geq t_{\mathrm{J}}\}},
\qquad
  {\omega}_{\text{jump}}(t)={\omega}\!\left(h_{\mathrm{J}}(t)\right).
\]
Here, $t_{\mathrm{J}}$ is the jump location and
$\Delta_{\mathrm{J}}$ controls its size. We used
$\Delta_{\mathrm{J}}=0.02$ and $0.05$. For each subject, we sampled one jump
location between observation times in $(0.30,0.70)$ and used that same location
for both jump sizes. The two conditions therefore differed only in the jump
size. The resulting warp remains nondecreasing and satisfies ${\omega}_{\text{jump}}(0)=0$ and
${\omega}_{\text{jump}}(1)=1$, but it is discontinuous.

\subsection{Simulation without phase variation}
\label{app:simulation-b0}

We set \(B=0\) in the trivariate simulation, so the true warp is the identity,
while retaining \(\rho\in\{0.2,0.6\}\) and fixing \(K=3\).
Table~\ref{tab:simulation-b0} compares AP-FAE with FAE and FAEclust using ACC,
ATV, and MSE. AP-FAE and FAE achieve similar ACC; their MSEs are equal at
\(\rho=0.2\), while AP-FAE attains lower MSE at \(\rho=0.6\).
Figure~\ref{fig:simulation-b0-rho060} illustrates the latter setting.

\begin{table}[!htbp]
\centering
\small
\setlength{\tabcolsep}{5pt}
\renewcommand{\arraystretch}{1.08}
\begin{tabular*}{\linewidth}{@{\extracolsep{\fill}}llccc@{}}
\toprule
Settings & Metric & FAE & FAEclust & AP-FAE \\
\midrule
$\rho=0.2, B=0$ & ACC $\uparrow$ & 0.950 & 0.729 & 0.946 \\
 & ATV $\downarrow$ & -- & -- & 1.164 \\
 & MSE $\downarrow$ & \textbf{0.015} & 0.354 & \textbf{0.015} \\
\addlinespace
$\rho=0.6, B=0$ & ACC $\uparrow$ & \textbf{0.916} & 0.766 & \underline{0.915} \\
 & ATV $\downarrow$ & -- & -- & 1.139 \\
 & MSE $\downarrow$ & \underline{0.013} & 0.354 & \textbf{0.011} \\
\bottomrule
\end{tabular*}
\caption{Trivariate simulation with \(B=0\). Bold and underline mark the best
and second-best distinct values when multiple methods report a metric; ``--''
denotes not applicable.}
\label{tab:simulation-b0}
\end{table}

\subsection{Scalability}
\label{app:scalability}

We examine the effects of sample size and grid resolution on AP-FAE in the trivariate setting with $B=0.75$ and $\rho=0.2$. Table~\ref{tab:sample-size-comparison} shows that increasing the sample size reduces reconstruction error overall, while clustering accuracy and alignment remain relatively stable. Table~\ref{tab:grid-size-comparison} shows only modest changes in reconstruction error and clustering accuracy across grid resolutions. The increase in raw ATV largely reflects its approximately $\sqrt{M}$-dependent scale (see \eqref{eq:app-metric-atv}); ATV normalized by $\sqrt{M}$ remains between $0.1160$ and $0.1187$. Training times for $1500$ epochs on one NVIDIA RTX~4090 (batch size $32$) include per-epoch reconstruction checks and exclude file I/O. Over the tested range, time grows approximately linearly with sample size but changes little with grid resolution.

\begin{table}[!htbp]
\centering
\small
\setlength{\tabcolsep}{9pt}
\begin{tabular}{rrcccr}
\toprule
Per class & Total $N$ & MSE $\downarrow$ & ACC $\uparrow$ & ATV $\downarrow$ & Time (s) \\
\midrule
30     & 90     & 0.016             & 0.949             & \underline{1.142} & 19.62 \\
50     & 150    & 0.014             & \textbf{0.963}    & \textbf{1.115}    & 31.19 \\
100    & 300    & 0.014             & 0.924             & 1.187             & 62.77 \\
1,000  & 3,000  & \underline{0.011} & \underline{0.953} & 1.166             & 555.09 \\
10,000 & 30,000 & \textbf{0.011}    & 0.952             & 1.180             & 5,422.98 \\
\bottomrule
\end{tabular}
\caption{AP-FAE across sample sizes for $D=3$, $B=0.75$, $\rho=0.2$, and $M=100$. Bold and underline indicate the best and second-best distinct values for MSE, ACC, and ATV.}
\label{tab:sample-size-comparison}
\end{table}

\begin{table}[!htbp]
\centering
\small
\setlength{\tabcolsep}{12pt}
\begin{tabular}{rcccr}
\toprule
Grid points $M$ & MSE $\downarrow$ & ACC $\uparrow$ & ATV $\downarrow$ & Time (s) \\
\midrule
30    & \textbf{0.013}    & \underline{0.942}    & \textbf{0.643}    & 60.82 \\
50    & \underline{0.013} & \textbf{0.952}             & \underline{0.830} & 60.33 \\
100   & 0.014           & 0.924             & 1.187             & 62.77 \\
1,000 & 0.013             & 0.935 & 3.670             & 65.09 \\
\bottomrule
\end{tabular}
\caption{AP-FAE across grid resolutions for $D=3$, $B=0.75$, $\rho=0.2$, and $100$ subjects per class. Bold and underline indicate the best and second-best distinct values for MSE, ACC, and ATV.}
\label{tab:grid-size-comparison}
\end{table}

\newpage

\subsection{Simulation visualization}
\label{app:simu-visu}

\begin{figure}[!htbp]
\centering
\includegraphics[width=\linewidth]{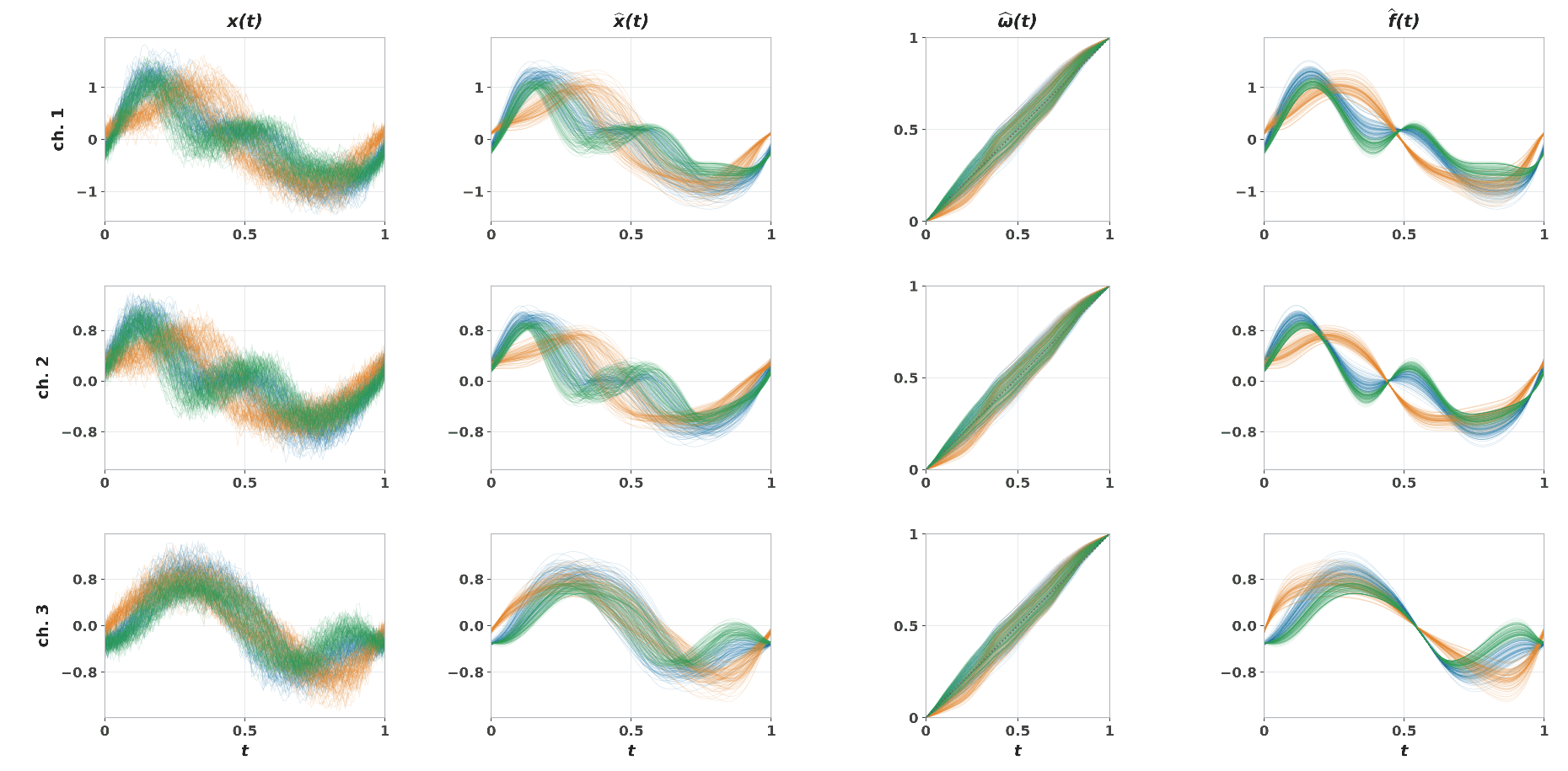}
\caption{Trivariate AP-FAE fit ($\rho=0,2, B=0.75$). Rows are channels; colours indicate classes and dotted lines denote identity.}
\label{fig:sim-d3-workflow}
\end{figure}

\begin{figure}[!htbp]
\centering
\includegraphics[width=\linewidth]{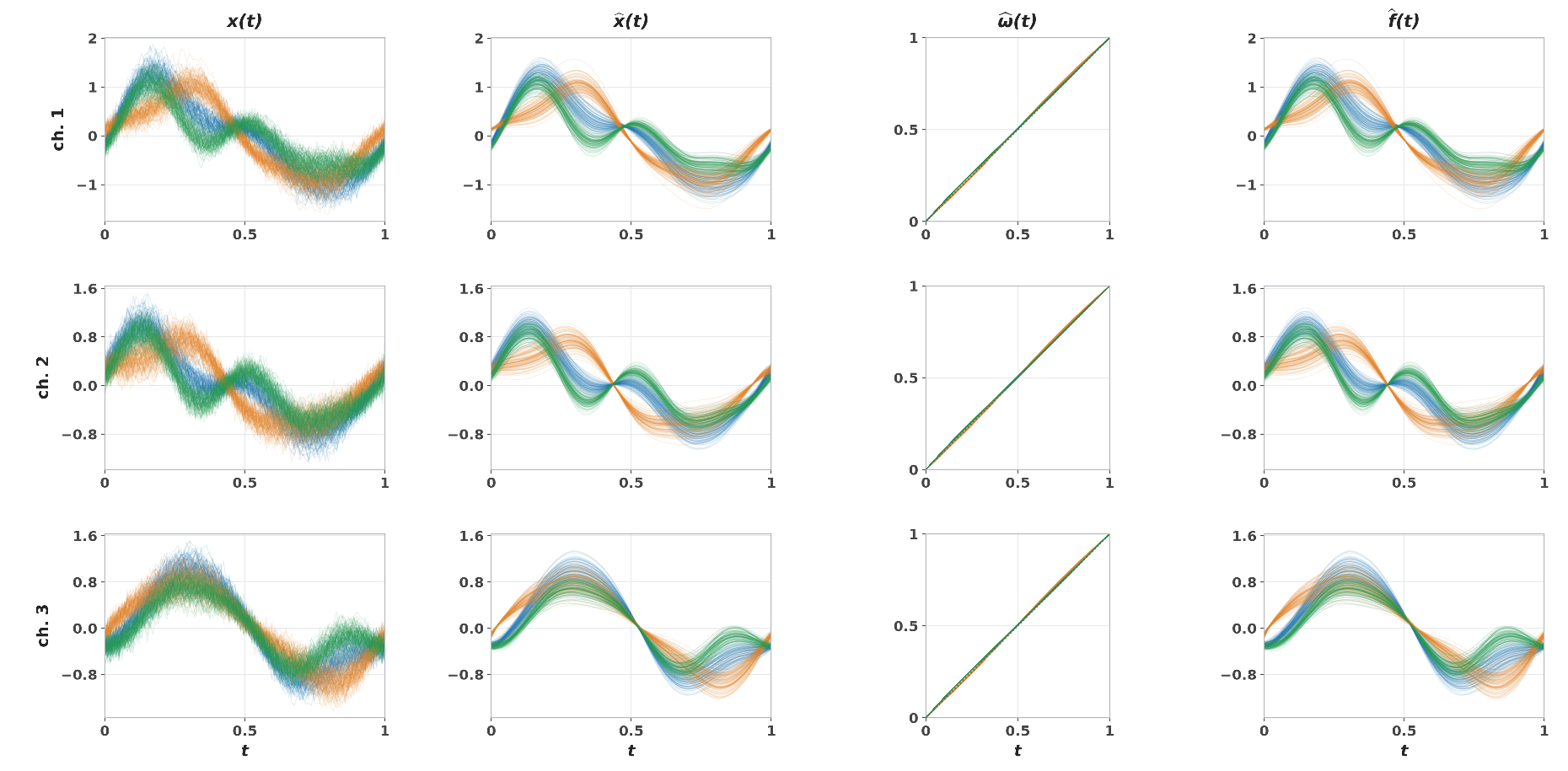}
\caption{AP-FAE fit with $B=0$ and $\rho=0.6$.}
\label{fig:simulation-b0-rho060}
\end{figure}
\FloatBarrier

\newpage

\subsection{Real-data visualizations}
\label{app:real-visualizations}

The following figures visualize all the real datasets. Columns show observed curves, reconstructions, learned inverse warps, and aligned curves. Colours indicate classes; dotted lines denote identity. 

\begin{figure}[!htbp]
\centering
\includegraphics[width=\linewidth]{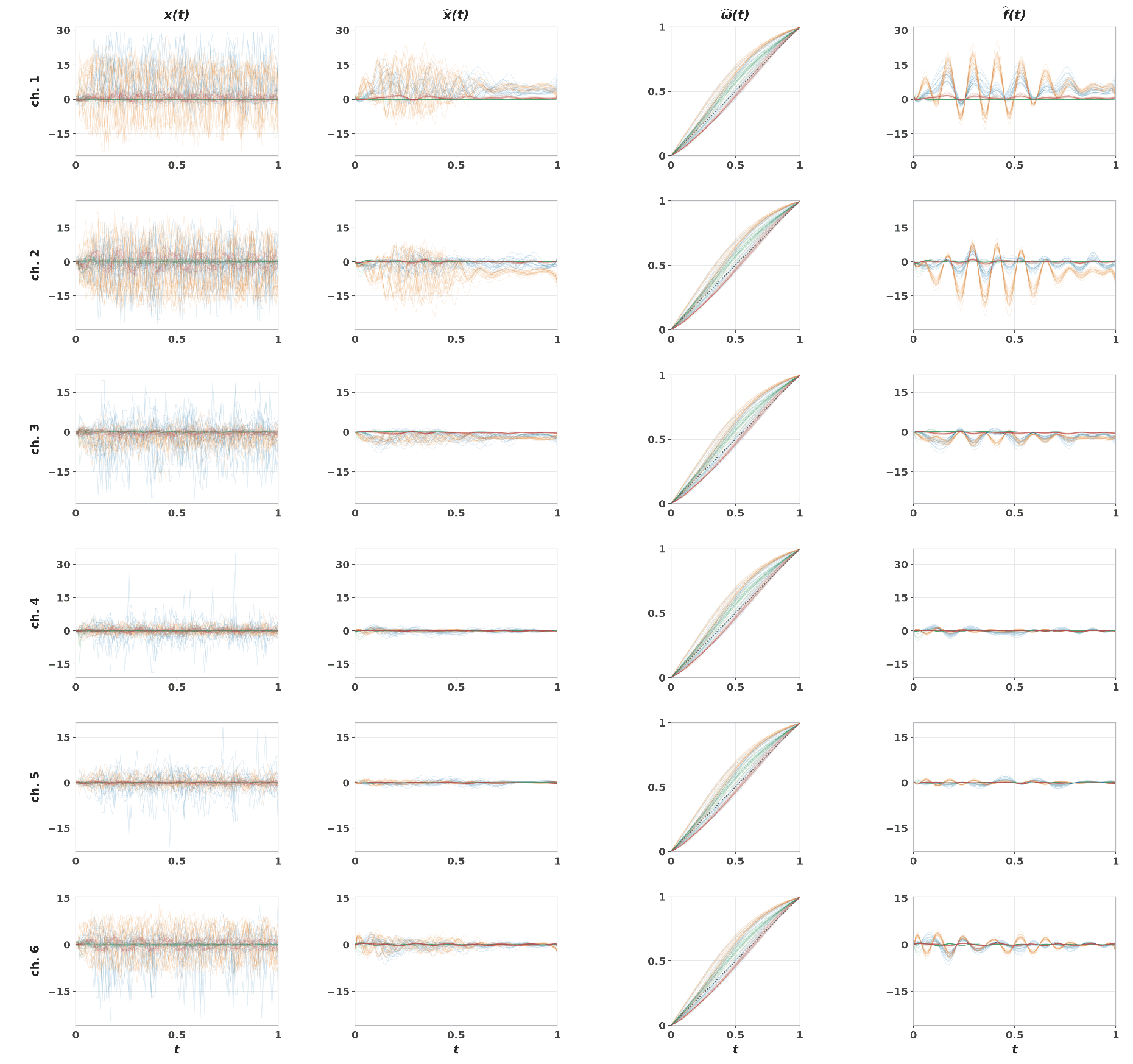}
\caption{AP-FAE fit on BasicMotions ($D=6$)}
\label{fig:realdata-basicmotions-workflow}
\end{figure}

\begin{figure}[!htbp]
\centering
\includegraphics[width=\linewidth]{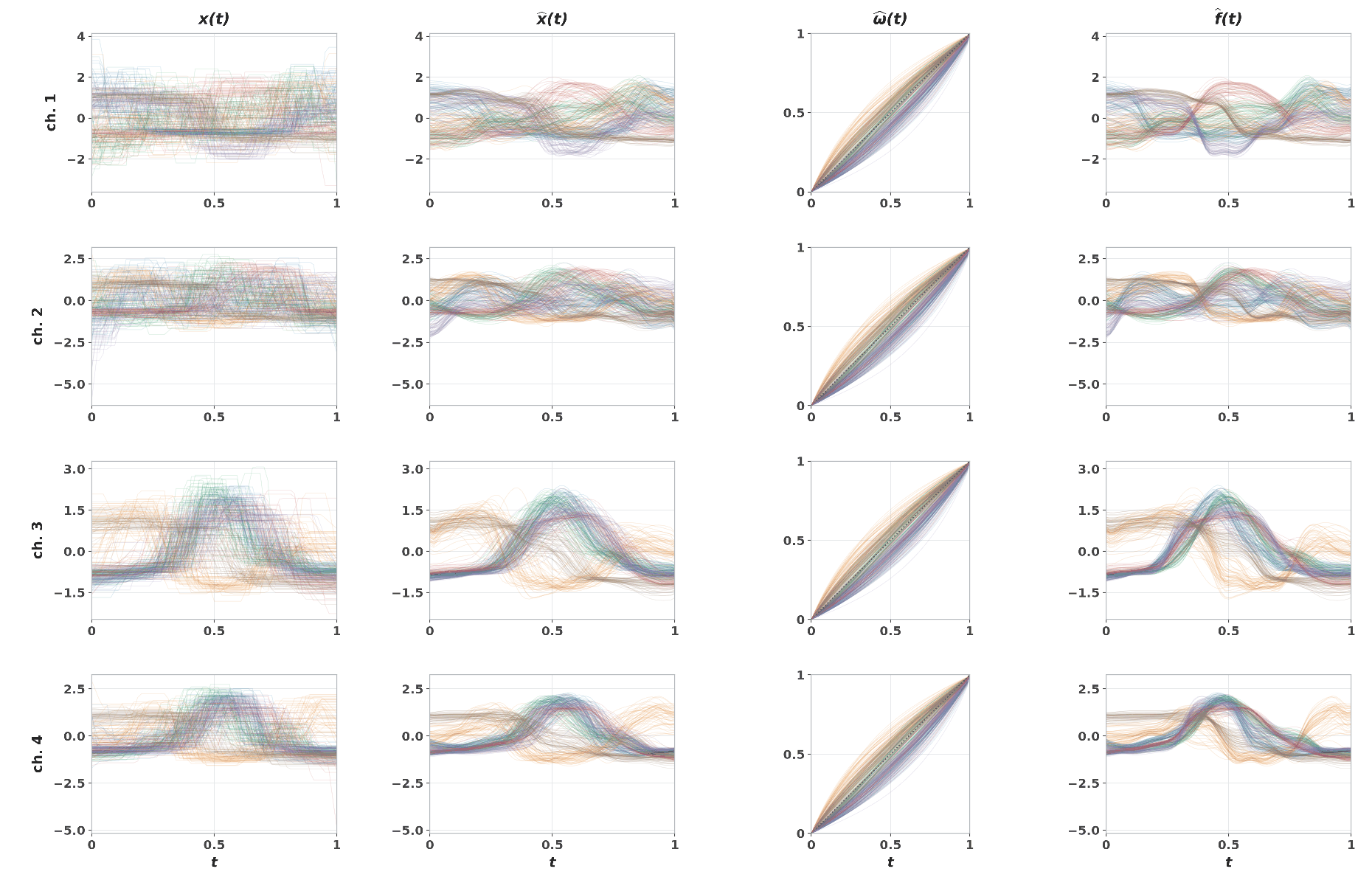}
\caption{AP-FAE fit on ERing ($D=4$)}
\label{fig:realdata-ering-workflow}
\end{figure}

\begin{figure}[!htbp]
\centering
\includegraphics[width=\linewidth]{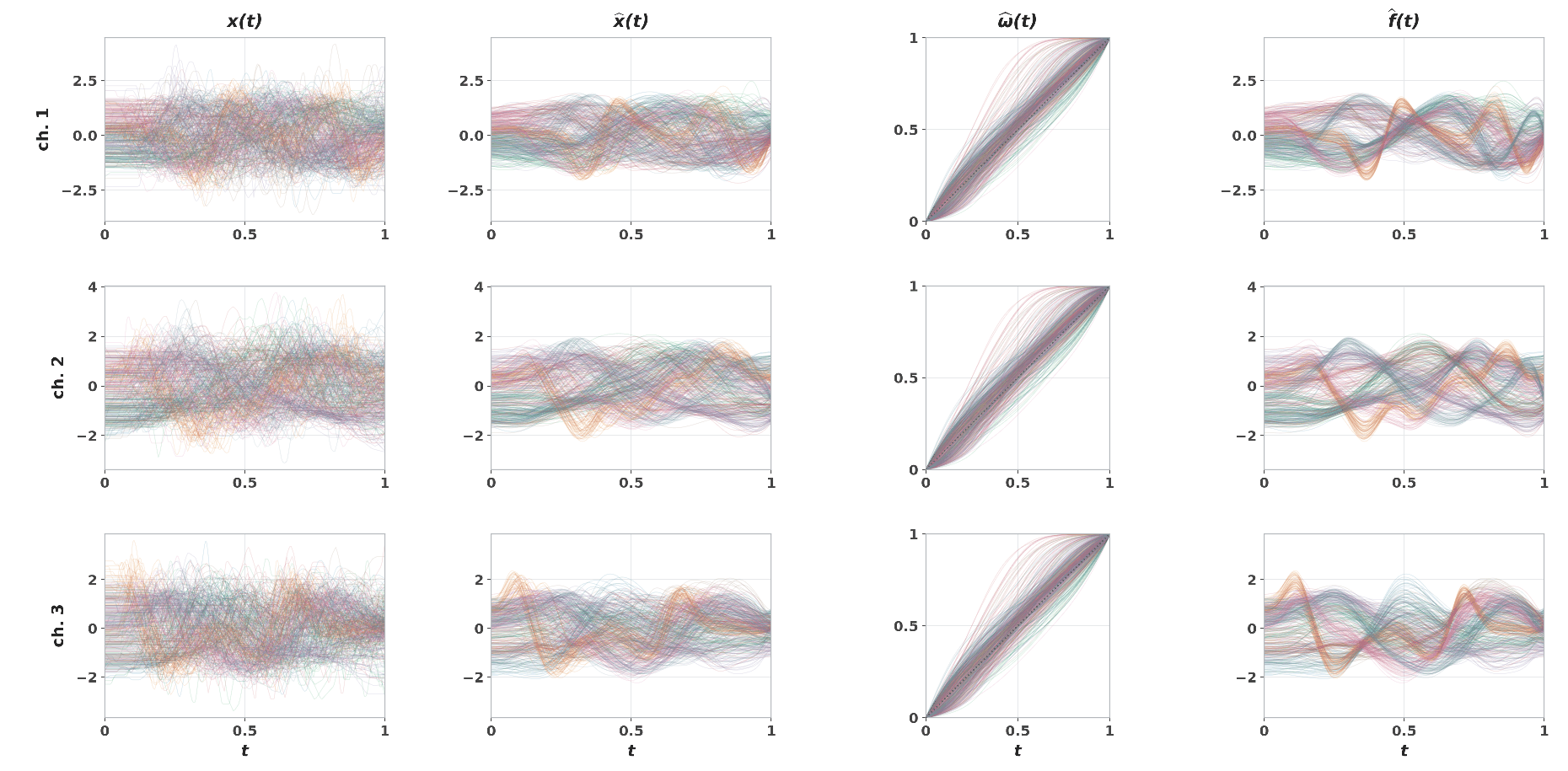}
\caption{AP-FAE fit on UWaveGestureLibrary ($D=3$)}
\label{fig:uwave}
\end{figure}

\begin{figure}[!htbp]
\centering
\includegraphics[width=\linewidth]{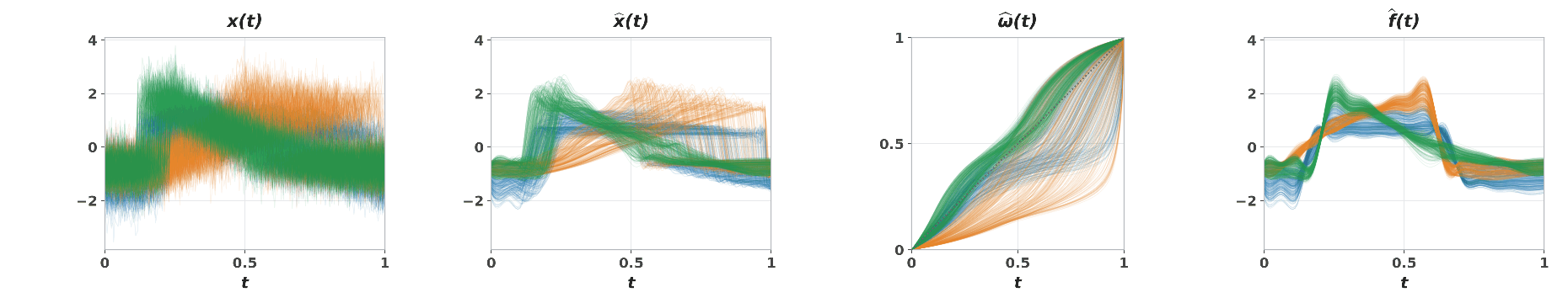}
\caption{AP-FAE fit on CBF ($D=1$).}
\label{fig:realdata-cbf-workflow}
\end{figure}

\begin{figure}[!htbp]
\centering
\includegraphics[width=\linewidth]{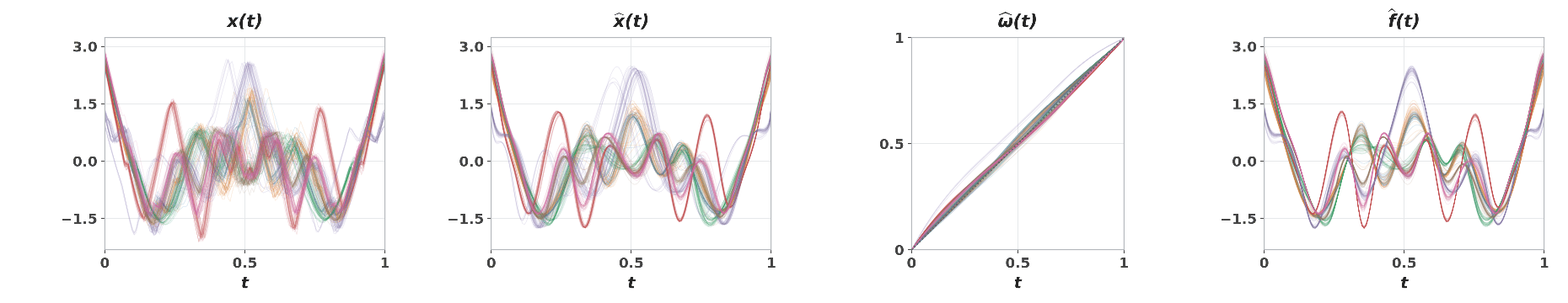}
\caption{AP-FAE fit on Plane ($D=1$).}
\label{fig:realdata-plane-workflow}
\end{figure}

\begin{figure}[!htbp]
\centering
\includegraphics[width=\linewidth]{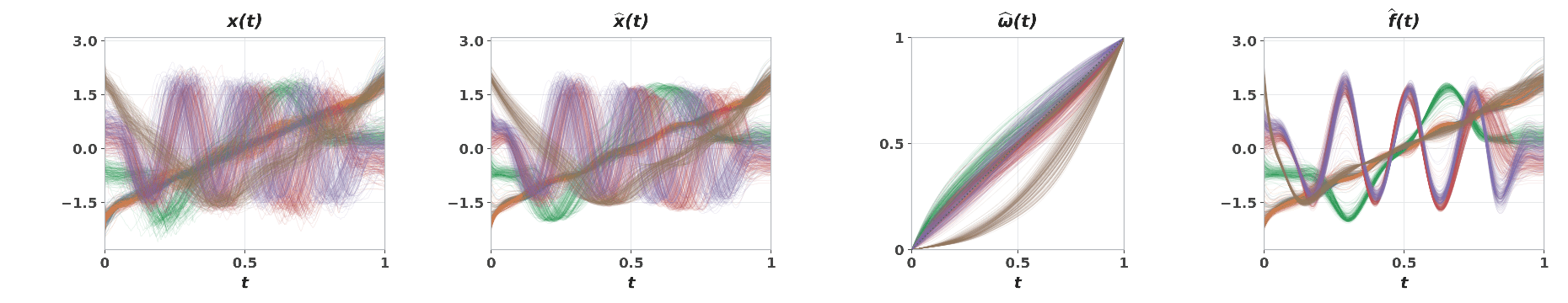}
\caption{AP-FAE fit on Symbols ($D=1$).}
\label{fig:realdata-symbols-workflow}
\end{figure}

\end{document}